\documentclass[preprint,12pt]{elsarticle}
\usepackage[utf8x]{inputenc}
\usepackage{silence}
\usepackage{pdfx}
\usepackage{listings}
\usepackage{hyperref}
\usepackage{xurl}
\usepackage{float}
\usepackage{changepage}
\usepackage{graphicx}
\usepackage{mathtools}
\usepackage{amsmath}
\usepackage{amssymb}
\usepackage{amsthm}
\usepackage{bm}
\usepackage{multirow}
\usepackage{multicol}
\usepackage{algorithm}
\usepackage{algpseudocode}
\usepackage[nohyperlinks]{acronym}
\makeatletter
\let\mhpc@acrodef\acrodef
\renewcommand*\acrodef[1]{%
  \@ifnextchar[{\mhpc@acrodef@opt{#1}}{\mhpc@acrodef@opt{#1}[#1]}%
}
\def\mhpc@acrodef@opt#1[#2]#3{%
  \newacro{#1}[#2]{#3}%
  \mhpc@acrodef{#1}[#2]{#3}%
}
\makeatother
\usepackage{booktabs}
\usepackage{array}
\usepackage{tikz}
\usetikzlibrary{arrows.meta,calc,positioning}
\usepackage[stretch=10,shrink=10]{microtype}
\usepackage[font=footnotesize,labelfont=bf,labelsep=period]{caption}
\newcommand{\M}{\mathcal{M}}
\newcommand{\U}{\mathcal{U}}
\newcommand{\PP}{\mathcal{P}}
\newcommand{\muh}{\hat{\bm{\mu}}}
\newcommand{\tabrowsep}{\specialrule{0.08pt}{1pt}{1pt}}

\newtheorem{lemma}{Lemma}[section]
\acrodef{CV}[CV]{Computer Vision}
\acrodef{ML}[ML]{Machine Learning}
\acrodef{DL}[DL]{Deep Learning}
\acrodef{AD}[AD]{Anomaly Detection}
\acrodef{MH}[MH]{Mahalanobis}
\acrodef{MHPC}[MH-PatchCore]{Mahalanobis PatchCore}
\acrodef{AE}[AE]{Autoencoder}
\acrodef{ViT}[ViT]{Vision Transformer}
\acrodef{PCA}[PCA]{Principal Component Analysis}
\acrodef{IPCA}[IPCA]{Incremental Principal Component Analysis}
\acrodef{RFF}[RFF]{Random Fourier Features}
\acrodef{SPD}[SPD]{Symmetric Positive Definite}
\acrodef{OAS}[OAS]{Oracle Approximating Shrinkage}
\acrodef{RBLW}[RBLW]{Rao-Blackwell Ledoit-Wolf}
\acrodef{NN}[NN]{Neural Network}
\acrodef{kNN}[kNN]{k-Nearest Neighbours}
\acrodef{GAN}[GAN]{Generative Adversarial Network}
\acrodef{AUROC}[AUROC]{Area Under the Receiver Operating Characteristic Curve}
\acrodef{GeoReS}[GeoReS]{Geometric Residual Selection}
\acrodef{SPADE}[SPADE]{Semantic Pyramid Anomaly Detection}
\acrodef{PANDA}[PANDA]{Adapting Pretrained Features for Anomaly Detection and Segmentation}
\acrodef{PaDiM}[PaDiM]{Patch Distribution Modeling}
\acrodef{DRAEM}[DRAEM]{Discriminatively Trained Reconstruction Embedding}
\acrodef{GRDNet}[GRD-Net]{Generative-Reconstructive-Discriminative Network}
\acrodef{GLASS}[GLASS]{Global and Local Anomaly co-Synthesis Strategy}
\acrodef{MVTecAD}[MVTec AD]{MVTec Anomaly Detection}
\acrodef{BFS}[BFS]{Blow-Fill-Seal}
\acrodef{RAM}[RAM]{Random-Access Memory}
\acrodef{FAISS}[FAISS]{Facebook Artificial Intelligence Similarity Search}
\acrodef{ROI}[ROI]{Region of Interest}
\acrodef{CPU}[CPU]{Central Processing Unit}

\journal{Engineering Applications of Artificial Intelligence}
\date{}
\begin{document}

\begin{frontmatter}

\title{Mahalanobis PatchCore:\\
Covariance-Aware and Streaming-Compatible\\
Industrial Anomaly Detection}

\author[unife,nais]{Niccolò Ferrari}
\ead{niccolo.ferrari@unife.it}

\author[unife,bonfiglioli]{Oligert Osmani}
\ead{oligert.osmani@edu.unife.it}

\author[unife]{Evelina Lamma}

\affiliation[unife]{organization={Department of Engineering, University of Ferrara},
                    addressline={Via Saragat 1},
                    postcode={44122},
                    city={Ferrara},
                    country={Italy}}

\affiliation[nais]{organization={NAIS s.r.l.},
                   addressline={Via Maria Callas 4},
                   postcode={40131},
                   city={Bologna},
                   country={Italy}}

\affiliation[bonfiglioli]{organization={Bonfiglioli Engineering S.r.l.},
                          addressline={Via Amerigo Vespucci 20},
                          postcode={44124},
                          city={Ferrara},
                          country={Italy}}

\acresetall
\begin{abstract}
Industrial visual anomaly detection is usually one-class: normal images are
abundant, while defects are rare, heterogeneous, and often unavailable during
system design.
PatchCore-style retrieval suits this setting because it scores test images from
a memory bank of normal patch features, but the standard Euclidean geometry
ignores feature correlations and its offline construction materialises the full
patch pool before subsampling.

We introduce \textbf{Mahalanobis PatchCore}, a covariance-aware,
streaming-compatible extension of PatchCore.
Its artificial intelligence contribution is a retrieval detector that estimates a
regularised covariance model in reduced feature space and whitens embeddings, so
Euclidean nearest-neighbour search after transformation implements Mahalanobis
retrieval.
A bounded-memory, re-iterable training pipeline builds the memory bank without
storing all normal patches at once, using incremental dimensionality reduction,
online covariance estimation, and streaming aggregation.

The engineering application is automated industrial inspection, where visual
anomaly detection must remain accurate under practical memory limits.
We evaluate the method on a public 15-category industrial anomaly-detection
benchmark and three industrial datasets covering blow-fill-seal strip-ampoule
meniscus inspection, amber-glass-ampoule bottom inspection, and lyophilised-cake
vial inspection.
Mahalanobis PatchCore preserves most offline PatchCore image-level performance
on the public benchmark while reducing peak memory from \(5.41\) to \(2.78\)~GB,
and improves the selected industrial mean image area under the receiver
operating characteristic curve from \(0.981\) to \(0.986\).
\end{abstract}


\end{frontmatter}

\noindent\textbf{Code:} \url{https://github.com/NickF93/MH-PatchCore}

\noindent\textbf{Data Availability:}
\ac{MVTecAD} is publicly available through its original
release~\cite{bergmann2019mvtec}. The industrial datasets used in this study
cannot be publicly released because they are derived from commercial inspection
data subject to non-disclosure restrictions.

\acresetall
\section{Introduction}
\label{sec:intro}

Industrial \ac{AD} is often a normality-modelling problem rather than a
standard supervised classification problem.
Large collections of normal products can usually be acquired from production,
whereas defective samples are scarce, visually heterogeneous, and sometimes not
available when an inspection system is designed.
At the same time, industrial deployment imposes practical constraints that are
not secondary to accuracy.
The detector should be stable, sufficiently diagnostic for engineering use, and
compatible with the memory footprint and runtime budgets of the target line.
Recent real-time and low-latency industrial \ac{AD} systems make this
requirement explicit: model complexity can limit real-time processing, while
strict runtime limits are tied to production rate and economic
viability~\cite{gudovskiy2022cflow,batzner2024efficientad}.
The central challenge is therefore to learn a reliable anomaly score from
normal data alone while keeping the inference mechanism simple enough to remain
operationally useful.

Patch-based retrieval methods are a natural fit for this setting because they
represent normality through local feature support rather than through an
explicit defect model.
Among them, \textbf{PatchCore}~\cite{Roth2021} has become a reference baseline:
a frozen backbone extracts patch descriptors, a coreset compresses the normal
memory bank, and test patches are scored by nearest-neighbour distance to that
bank.
This design is attractive because the learned model is essentially a geometry
over patch embeddings plus a fixed support set.
It is also precisely this simplicity that makes its limitations visible.
First, the standard Euclidean retrieval geometry implicitly treats feature
directions as equally scaled and uncorrelated, although deep patch embeddings
are often anisotropic and correlated.
Second, the usual offline construction materialises the complete normal patch
pool before subsampling it, so peak training memory follows the number and
dimensionality of extracted patches rather than the size of the final support
set.
This may be acceptable for small benchmark categories, but it is less
satisfactory when image counts, resolutions, or industrial deployment
constraints grow.

This paper introduces \textbf{\ac{MHPC}}, a
covariance-aware and streaming-compatible extension of PatchCore.
The method keeps the retrieval-based inference paradigm intact, but changes the
space in which retrieval is performed and the way the memory bank is built.
Patch embeddings are reduced and whitened using a regularised covariance model,
so Euclidean nearest-neighbour search in the transformed space implements
Mahalanobis retrieval in the reduced coordinates.
Training is organised as a bounded-memory, re-iterable pipeline: the method
fits the reduction, estimates the covariance statistics, and constructs the
memory bank without requiring the complete patch pool to reside in memory at
once.
Thus, the proposal is not evaluated here as an autonomous continual-learning
system.
Rather, \acs{MHPC} is a PatchCore-style detector whose bounded-memory
training procedure consumes normal data in mini-batches. Additional curated
normal batches can be included while the detector is still being fitted, before
the reduction, covariance transform, and memory bank are finalised for
deployment.
In the reported experiments, however, the reduction, covariance transform, and
memory bank are finalised before inference.

The paper makes three main contributions:
\begin{enumerate}
  \item \textbf{Covariance-aware Mahalanobis scoring in whitened space.}
        We replace isotropic nearest-neighbour retrieval with an equivalent
        Euclidean search in a whitened reduced space, preserving PatchCore's
        support-set inference while making distances sensitive to feature
        covariance.

  \item \textbf{Bounded-memory streaming-compatible training.}
        We formulate training as a multi-pass but bounded-memory procedure that
        fits the reducer, estimates covariance statistics, and constructs the
        support set without materialising all normal patches simultaneously.

  \item \textbf{Empirical evaluation of detection quality and resource
        footprint.}
        We compare the proposed geometry and memory-bank construction against
        the original PatchCore baseline, Euclidean streaming controls, and
        selected budget/constructor ablations on \ac{MVTecAD} and selected
        industrial inspection datasets, reporting both detection quality and
        resource telemetry.
\end{enumerate}

We evaluate the method in two complementary settings.
The public component is \acs{MVTecAD} \cite{bergmann2019mvtec}, considered through
its 15 category-specific datasets; the industrial component consists of three
image-level inspection datasets covering meniscus, ampoule-bottom, and
lyophilised-product inspection.
This separation keeps the benchmark and deployment questions distinct.
On \acs{MVTecAD}, where pixel masks are available, pixel-wise metrics are reported
for localisation completeness.
On the industrial datasets, where masks are not available and the operational
decision is image-level, the comparison focuses on image-level metrics and
\acs{RAM} telemetry, with a Meniscus subset ladder used to study
\acs{RAM}-constrained training
budgets.

The remainder of the paper is organised as follows.
Section~\ref{sec:related} reviews the main lines of related work.
Section~\ref{sec:method} formalises \acs{MHPC} and its training/inference
pipeline.
Section~\ref{sec:experiments} describes the experimental protocol and the
evaluated configurations.
Section~\ref{sec:results} presents category-resolved \acs{MVTecAD} results together
with the industrial image-level comparison and the MVTec-only pixel-level
completeness results.
Sections~\ref{sec:discussion} and~\ref{sec:conclusion} discuss the findings and
conclude the paper.

\section{Related Work}
\label{sec:related}

\subsection{Embedding-Based and Memory-Bank Anomaly Detection}

Embedding-based \ac{AD} methods avoid learning an explicit image reconstruction
model. Instead, they represent normality in a feature space and detect test
samples whose local descriptors fall outside the nominal support. This idea is
particularly effective in industrial inspection, where strong pretrained
features often provide useful local descriptors even when only defect-free
training images are available.
\ac{SPADE}~\cite{cohen2020spade} showed that nearest-neighbour matching on
pretrained feature pyramids can support sub-image anomaly localisation with
minimal task-specific training.
\acs{PANDA}~\cite{reiss2021panda} further studied pretrained representations in
one-class anomaly detection and segmentation, highlighting their strength as
well as the risk of feature collapse when they are adapted to the normal class.
\textbf{\ac{PaDiM}}~\cite{Defard2021} models each spatial location of the feature map
with a multivariate Gaussian and scores deviations with a Mahalanobis distance.
It therefore makes covariance structure explicit, but ties the statistical model
to aligned feature-map positions.
\textbf{PatchCore}~\cite{Roth2021} takes a different route: it stores nominal
patch descriptors in a global memory bank, compresses that bank with greedy
coreset selection, and performs nearest-neighbour retrieval in Euclidean space.
This makes PatchCore a strong reference point for one-class industrial visual
inspection, but its standard scoring geometry treats feature directions as
isotropic after embedding extraction.
\acs{MHPC} is positioned at this interface. It keeps the PatchCore memory-bank
view of normal support, while introducing a global covariance-aware whitening
transform before retrieval.

\subsection{Reconstruction, Distillation, and Synthesis}

Reconstruction-based approaches learn to reproduce nominal images and interpret
large reconstruction residuals as evidence of abnormality.
\ac{GAN}-based methods such as GANomaly~\cite{Akccay2018} and
Skip-GANomaly~\cite{Akccay2019} use adversarial encoder--decoder objectives to
model the normal image manifold.
Feature-reconstruction methods move the reconstruction target from pixels to
pretrained representations; Reverse Distillation~\cite{deng2022reverse}, for
example, trains a decoder to recover multi-scale teacher features from a compact
one-class embedding.
\ac{DRAEM}~\cite{zavrtanik2021draemdiscriminativelytrained} shifts this family toward
discriminative surface inspection by training with synthetic texture anomalies
and learning a joint reconstructive-discriminative representation.
Related industrial methods, such as
\ac{GRDNet}~\cite{ferrari2023grdnet,ferrari2024integration}, combine
generative, reconstructive, and discriminative modules with task-specific
attention mechanisms, while recent anomaly-synthesis work such as
\ac{GLASS}~\cite{chen2024unifiedanomalysynthesisstrategy} studies how to generate
broader and more controllable synthetic defects.
These methods are complementary to the present work: they modify the training
objective or the anomaly-generation process, whereas \acs{MHPC} keeps the
frozen-feature retrieval paradigm and asks how its retrieval geometry and memory
bank can be made covariance-aware and memory-bounded.

\subsection{Mahalanobis Distance in Anomaly Detection}

Mahalanobis distance was introduced as a generalised statistical distance by
Mahalanobis~\cite{mahalanobis1936generalised}.
Covariance-based distances have since become a classical tool for multivariate
outlier and leverage diagnostics~\cite{rousseeuw1990unmasking}, and their role
as a general multivariate distance has been reviewed in chemometrics, pattern
recognition, and process control~\cite{demaesschalck2000mahalanobis}.
In visual \ac{AD}, Mahalanobis scoring has been reintroduced to account
for feature correlations in high-dimensional embedding spaces.
Rippel et al.~\cite{rippel2021gaussianad} model normal images with a global
multivariate Gaussian in pretrained feature space, using Mahalanobis distance as
the anomaly score.
\acs{PaDiM}~\cite{Defard2021} applies per-location Mahalanobis distance to model the
normal distribution of patch embeddings, computing a separate covariance matrix
per spatial position.
Mahalanobis scoring has also been studied with hybrid transformer and
convolutional patch representations, where multivariate Gaussian models are fit
to normal embeddings and test patches are scored by their distance to those
models~\cite{dini2023visualanomalydetection}.
Our use of Mahalanobis geometry is deliberately different. Rather than replacing
the memory bank with separate local Gaussian classifiers, \acs{MHPC} estimates
a single regularised covariance over the reduced global patch population and
uses the resulting whitening map to change the metric of PatchCore-style
nearest-neighbour retrieval.

\subsection{Streaming-Compatible Training}

Memory-bank methods inherit a scaling issue from their main strength: dense
patch-level coverage of the nominal distribution can require storing a very
large number of descriptors. PatchCore mitigates this pressure through greedy
coreset subsampling~\cite{Roth2021}, but the standard formulation still builds
the candidate patch pool before selecting the representative subset.
This distinction matters when the training set is large or when memory telemetry
is part of the evaluation itself.
Resource-aware visual anomaly detection has also been studied from the inference
side: EfficientAD~\cite{batzner2024efficientad} targets millisecond-level
latencies by combining a lightweight feature extractor, student--teacher
prediction, and a compact global-analysis branch.
That line of work is complementary to ours: it optimises inference architecture
and throughput, whereas this paper studies the bounded-memory construction of a
PatchCore-style support set.
In this paper, ``streaming'' therefore refers to a bounded-memory,
streaming-compatible training procedure over a re-iterable nominal dataset, not
to continual one-pass adaptation.
The contribution is the joint treatment of incremental dimensionality reduction,
streaming covariance estimation, and bounded-memory support-set construction
inside a PatchCore-style retrieval pipeline.

\section{Method}
\label{sec:method}

\acs{MHPC} learns four objects from normal data:
\((R,\muh,L,\M)\).
The map \(R\) places patch descriptors in a reduced coordinate system, while
\(\muh\) and \(L\) define the covariance-aware whitening transform used for
retrieval.
The bank \(\M\) is the bounded support set queried at inference time.
The following subsections first fix the PatchCore baseline notation, then
formalise covariance-aware retrieval, bounded-memory bank construction, the
two-pass \ac{GeoReS} ablation, and image scoring.
Appendix~\ref{app:symbols} provides a compact notation reference for the
symbols used throughout the method.

Figure~\ref{fig:mhpatchcore-overview} gives a high-level view of the training
and inference flow before the individual components are formalised.

\begin{figure}[htbp]
  \centering
  \includegraphics[width=\textwidth]{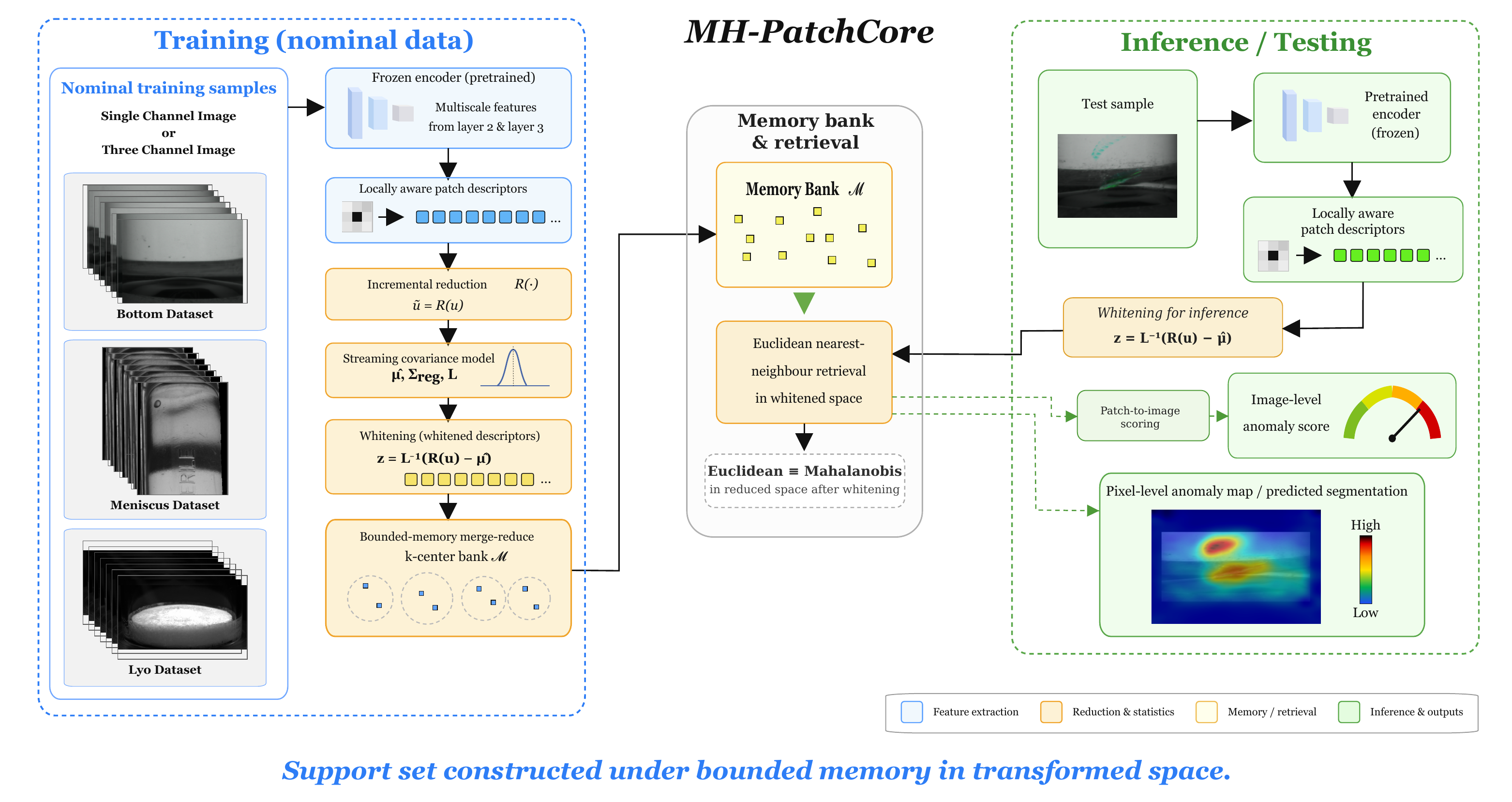}
  \caption{Overview of \acs{MHPC} training and inference. Nominal
  single-channel or three-channel inspection images are mapped by a frozen
  encoder to locally aware patch descriptors, reduced incrementally, whitened
  by a regularised covariance model, and summarised into a bounded-memory
  \(k\)-center bank. At inference, the same encoder, reduction, and whitening
  transform are applied before Euclidean nearest-neighbour retrieval in the
  whitened space, which is equivalent to Mahalanobis retrieval in the reduced
  coordinates. The detector returns image-level anomaly scores and spatial
  anomaly maps with predicted segmentations.}
  \label{fig:mhpatchcore-overview}
\end{figure}

\subsection{Problem Setup and PatchCore Formulation}
\label{sec:method:patchcore}

We use the following notation to describe the detector independently of whether
the normal patch pool is materialised offline or consumed under a
bounded-memory streaming regime.
We consider the standard one-class anomaly-detection setting, where the
training set
\begin{equation}
  \mathcal{D}_{\mathrm{train}}=\{x_i\}_{i=1}^{N},
  \qquad x_i \sim p_{\mathrm{normal}},
\end{equation}
contains only normal images sampled from the normal production distribution.
At test time, the detector assigns each image a real-valued anomaly score
\begin{equation}
  s:\mathcal{X}\to\mathbb{R},
\end{equation}
with larger values of \(s(x)\) indicating stronger evidence of departure from
the normal class.
A frozen feature extractor together with a patch-construction operator maps each
image $x$ to a finite set of patch descriptors
\begin{equation}
  \PP(x)=\{u^{(j)}\}_{j=1}^{n_x}, \qquad u^{(j)}\in\mathbb{R}^{d_0}.
\end{equation}
The full pool of normal training patches is therefore
\begin{equation}
  \U=\bigcup_{x\in\mathcal{D}_{\mathrm{train}}}\PP(x).
\end{equation}
Here \(\U\) is a notional pool of descriptors: the offline reference can
materialise it, whereas the streaming variants below process its elements in
mini-batches.

PatchCore~\cite{Roth2021} represents normality by a memory bank
$\M\subseteq\U$ and performs anomaly detection by nearest-neighbour retrieval in
patch space.
In its classical form, this bank can be written as greedy coreset subsampling
with retained fraction \(\eta\),
\begin{equation}
  \M=\operatorname{GreedyCoreset}(\U,\eta), \qquad \eta\in(0,1],
\end{equation}
and, in the squared-distance convention used throughout this paper, each query
patch is scored by Euclidean nearest-neighbour distance:
\begin{equation}
  a(u)=\min_{m\in\M}\|u-m\|_2^2, \qquad u\in\PP(x).
\end{equation}
This convention differs from the unsquared form only by a monotone
transformation, so it preserves nearest-neighbour identities while fixing the
score scale used by the downstream scoring rule.
This formulation is attractive because inference depends only on a geometry over
patch embeddings and a fixed support set.
At the same time, it reveals the two limitations targeted in this paper.
The first is geometric: Euclidean retrieval ignores covariance structure in deep
patch representations.
The second is computational: the classical pipeline builds $\M$ only after
materialising the full patch pool $\U$, so peak training-time memory scales with
the number of extracted patches rather than with the final support-set size.

Figure~\ref{fig:patchcore-mh-concept} summarises the conceptual difference
between the offline PatchCore pipeline and the streaming covariance-aware
variant developed in this paper.
The right-hand side introduces the notation formalised in the following
subsections.

\begin{figure}[htbp]
  \centering
  \includegraphics[width=\textwidth]{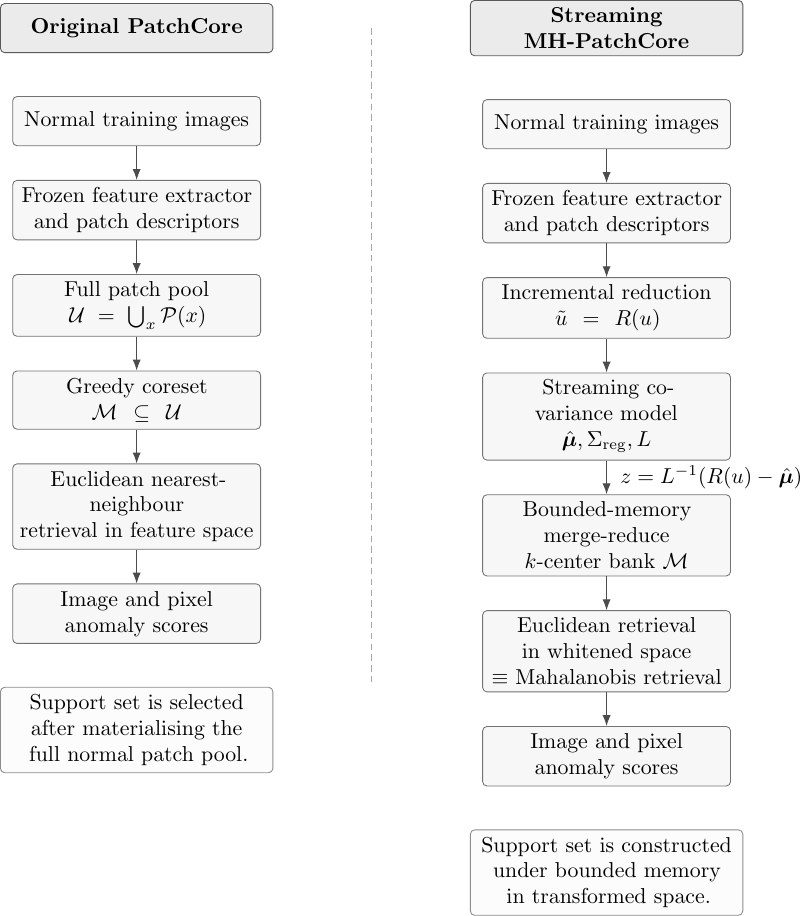}
  \caption{Conceptual comparison between the offline PatchCore baseline and the
  streaming covariance-aware variant studied in this paper. The diagram
  contrasts the point at which the support set is built and the geometry used
  for nearest-neighbour retrieval.}
  \label{fig:patchcore-mh-concept}
\end{figure}

\subsection{Covariance-Aware Retrieval in Reduced Space}
\label{sec:method:whitening-regularization}

\acs{MHPC} modifies PatchCore along two coupled axes: the geometry used for
retrieval and the procedure used to build the memory bank.
We first map each patch descriptor to a reduced coordinate system,
\begin{equation}
  \tilde{u}=R(u), \qquad R:\mathbb{R}^{d_0}\to\mathbb{R}^{k},
\end{equation}
where \(R\) denotes the fitted dimensionality reducer learned from normal
training patches.
In the \ac{PCA}/\ac{IPCA} instantiation used in the reported experiments, this
reducer is the centred projection
\begin{equation}
  R(u)=W^\top(u-\bar{u}_R),
\end{equation}
with \(W\in\mathbb{R}^{d_0\times k}\) denoting the learned projection matrix and
\(\bar{u}_R\) the reducer mean.
Working in reduced space serves two purposes.
It suppresses redundancy before covariance estimation and makes global
second-order modelling numerically tractable.

Let $\muh$ and $\hat{\Sigma}$ denote the mean and empirical covariance of the
reduced normal patches.
In the streaming setting these statistics are accumulated from mini-batches via
numerically stable moment updates.
The full streaming covariance update is summarised in
Appendix~\ref{app:welford}.
Given current state $(n,\muh,M_2)$ and batch summary
$(n_b,\muh_b,M_{2,b})$, we use
\begin{align}
  \bm{\delta} &= \muh_b-\muh, \\
  \muh &\leftarrow \muh + \frac{n_b}{n+n_b}\,\bm{\delta}, \\
  M_2 &\leftarrow M_2 + M_{2,b}
        + \frac{n\,n_b}{n+n_b}\,\bm{\delta}\bm{\delta}^\top, \\
  n &\leftarrow n+n_b,
\end{align}
followed by
\begin{equation}
  \hat{\Sigma}=\frac{M_2}{n-1}.
\end{equation}
The covariance is then regularised to a symmetric positive-definite matrix
$\Sigma_{\mathrm{reg}}$.
In the canonical baseline this regularisation is based on shrinkage toward the
scaled identity, optional eigenvalue flooring, and stable Cholesky
factorisation; the numerical details are given in
Appendix~\ref{app:covreg-alternatives}.

Let \(\delta\) be the adaptive Cholesky jitter from
Appendix~\ref{app:covreg-alternatives}, and define
\begin{equation}
  \Sigma_L = \Sigma_{\mathrm{reg}} + \delta I = LL^\top .
\end{equation}
We define the whitened embedding
\begin{equation}
  z(u)=L^{-1}\bigl(R(u)-\muh\bigr).
\end{equation}
For any two reduced embeddings $\tilde{u}$ and $\tilde{v}$,
\begin{equation}
  \|z(u)-z(v)\|_2^2
  =(\tilde{u}-\tilde{v})^\top \Sigma_L^{-1}(\tilde{u}-\tilde{v}),
\end{equation}
so Euclidean retrieval in whitened space is equivalent to Mahalanobis retrieval
in the reduced coordinates under the effective covariance \(\Sigma_L\).
A formal statement and proof of this equivalence are given in
Lemma~\ref{lem:whitening-mahalanobis} in
Appendix~\ref{app:whitening-equivalence}.
This is the central geometric step of \acs{MHPC}: the retrieval structure of
PatchCore is preserved, but the isotropic Euclidean geometry is replaced by a
covariance-aware one.

\subsection{Streaming Memory-Bank Construction}
\label{sec:method:memory-bank-construction}

The second contribution concerns how the memory bank is constructed.
Instead of materialising the entire patch pool and compressing it afterwards, we
seek a bounded support set directly in the transformed space.
Let
\begin{equation}
  Z=\{z_n\}_{n=1}^{N_Z}, \qquad z_n=z(u_n),
\end{equation}
denote the whitened stream of normal patches.
This notation refers to the sequence obtained by concatenating transformed
mini-batches; the streaming constructors do not require the full set to be
materialised.
The target object is a bank $\M$ of size at most $K$ that preserves coverage of
this set.
The merge-reduce constructor applies deterministic farthest-first coverage steps,
\begin{equation}
  z^\star\in\operatorname*{arg\,max}_{z\in Z}
  \min_{c\in C_\ell}\|z-c\|_2,\qquad
  C_{\ell+1}=C_\ell\cup\{z^\star\},
\end{equation}
as a greedy surrogate for the classical $k$-center objective.
This coverage principle retains representative normal patches for
nearest-neighbour retrieval.

Training is therefore cast as the map
\begin{equation}
  \mathcal{D}_{\mathrm{train}} \mapsto (R,\muh,L,\M),
\end{equation}
which records the mathematical detector state needed for scoring.
The construction decomposes into three conceptual stages.
First, the reducer $R$ is fit on the stream of patch embeddings.
Second, reduced patch statistics are accumulated to obtain
$(\muh,\Sigma_{\mathrm{reg}})$ and the whitening operator $L^{-1}$.
Third, the whitened patches are passed to a bounded-memory constructor
\begin{equation}
  \M=\mathcal{G}(Z;K).
\end{equation}
In the canonical baseline, $\mathcal{G}$ is instantiated by a merge-reduce
$k$-center strategy: local chunks are summarised by farthest-first
representatives, these summaries are recursively merged and reduced, and the
final bank is obtained without ever storing the full set $Z$ in memory.
The resulting bank contains representatives rather than centroids, which is the
appropriate object for retrieval-based inference.

This streaming formulation should be understood as a bounded-memory training
procedure, not as an online post-deployment adaptation rule.
Its purpose is to make PatchCore-style bank construction feasible when the full
normal patch set is too large to materialise comfortably in \acs{RAM}.

Algorithm~\ref{alg:mh-patchcore} summarises the canonical \acs{MHPC} training
and inference procedure.

\begin{algorithm}[H]
  \caption{Canonical \acs{MHPC} training and inference}
  \label{alg:mh-patchcore}
  \footnotesize
  \begin{algorithmic}[1]
    \Require Normal training set \(\mathcal{D}_{\mathrm{train}}\), query image
    \(x\), patch map \(\PP(\cdot)\)
    \Require Reducer policy \(R\), covariance regularisation policy, bank budget
    \(K\), bounded-memory constructor \(\mathcal{G}\)
    \Ensure Detector state \((R,\muh,L,\M)\) and image score \(s(x)\)
    \Statex \textbf{Training}
    \State Initialise the reducer fit for \(R\)
    \ForAll{normal-image mini-batches \(B\subset\mathcal{D}_{\mathrm{train}}\)}
      \State Extract patch descriptors \(u\in\PP(x_i)\) for all \(x_i\in B\)
      \State Update the reducer fit with the current patch mini-batch
    \EndFor
    \State Finalise the reducer \(R\)
    \State Initialise streaming moments \((n,\muh,M_2)\)
    \ForAll{normal-image mini-batches \(B\subset\mathcal{D}_{\mathrm{train}}\)}
      \State Compute reduced patches \(\tilde u=R(u)\) for all
      \(u\in\PP(x_i),\,x_i\in B\)
      \State Update \((n,\muh,M_2)\) with the reduced patch mini-batch
    \EndFor
    \State Form \(\hat{\Sigma}=M_2/(n-1)\), regularise it, and factor
    \(\Sigma_L=\Sigma_{\mathrm{reg}}+\delta I=LL^\top\)
    \State Initialise the bounded-memory constructor \(\mathcal{G}\)
    \ForAll{normal-image mini-batches \(B\subset\mathcal{D}_{\mathrm{train}}\)}
      \State Compute \(z(u)=L^{-1}(R(u)-\muh)\) for all
      \(u\in\PP(x_i),\,x_i\in B\)
      \State Update \(\mathcal{G}\) with the whitened patch mini-batch
    \EndFor
    \State Finalise the memory bank \(\M=\mathcal{G}(Z;K)\)
    \Statex \textbf{Inference}
    \State Compute \(z(u)=L^{-1}(R(u)-\muh)\) for all \(u\in\PP(x)\)
    \State Compute \(a(u)=\min_{m\in\M}\|z(u)-m\|_2^2\) for all
    \(u\in\PP(x)\)
    \State Compute \(s(x)\) using Eqs.~\eqref{eq:weight}--\eqref{eq:score}
    \State \Return \((R,\muh,L,\M),\,s(x)\)
  \end{algorithmic}
\end{algorithm}

\subsection{A Two-Pass Ablation: Geometric Residual Selection}
\label{sec:method:geores}

The canonical constructor above uses a single traversal of the transformed
patch stream at the memory-bank construction stage.
We also evaluate \ac{GeoReS}, a deliberately stronger but less strictly
streaming ablation that gives the constructor one additional pass.
The goal is not to change the scoring rule or the Mahalanobis geometry, but to
ask whether a coverage coreset improves when it is first allowed to estimate
where the provisional bank leaves geometric residuals.

Let $\mathcal{G}_K$ denote the merge-reduce \(k\)-center constructor with final
budget \(K\), and let \(Z\) again denote the whitened normal patch stream.
The first pass builds a provisional coverage set
\begin{equation}
  P=\mathcal{G}_K(Z).
\end{equation}
For any normal patch \(z\in Z\), its residual with respect to this provisional
bank is
\begin{equation}
  r_P(z)=\min_{p\in P}\|z-p\|_2 .
\end{equation}
Large values of \(r_P(z)\) identify normal regions that are poorly covered by
the provisional \(k\)-center solution.
These regions are geometric tails in the sense used here: they are not labelled
anomalies, and they are not selected from test data, but they occupy sparse or
under-represented parts of the normal embedding support.

In the second pass, the memory budget is split into a main coverage budget and
a residual-tail budget,
\begin{equation}
  \tilde K_0=\operatorname{round}(\alpha K),\qquad
  K_0=\min\{K-1,\max\{1,\tilde K_0\}\},\qquad
  K_t=K-K_0,
\end{equation}
with \(0<\alpha<1\), so both \(K_0\) and \(K_t\) remain positive.
A main bank is built as
\begin{equation}
  M_0=\mathcal{G}_{K_0}(Z),
\end{equation}
while a bounded candidate set keeps the \(q\) normal patches with largest
provisional residuals,
\begin{equation}
  T_q=\operatorname{Top}_{q}\{\,z\in Z:\ r_P(z)\,\}.
\end{equation}
The final \(K_t\) tail representatives are then chosen greedily from \(T_q\),
balancing large residuals with diversity relative to the main bank:
\begin{equation}
  t_j \in \operatorname*{arg\,max}_{z\in T_q\setminus\{t_1,\ldots,t_{j-1}\}}
  \min_{m\in M_0\cup\{t_1,\ldots,t_{j-1}\}}\|z-m\|_2 ,
  \quad j=1,\ldots,K_t .
\end{equation}
The resulting bank is the budgeted, deduplicated completion
\begin{equation}
  \M_{\mathrm{GeoReS}}
  =\Pi_K\!\left(M_0\cup\{t_1,\ldots,t_{K_t}\};Z\right).
\end{equation}
Here \(\Pi_K\) removes duplicate representatives, fills any remaining capacity
from the coverage pool, and truncates the bank to the target budget.

\ac{GeoReS} is therefore an informed \(k\)-center variant.
Its information remains purely unsupervised and normal-class: the first pass
only exposes where the normal support is weakly represented by a provisional
coverage set.
Because this requires two memory-bank passes instead of one, we report it as an
ablation rather than as the canonical \acs{MHPC} configuration.

\subsection{Inference and Image Scoring}
\label{sec:method:inference}

At inference time, a test image $x$ is mapped to its patch set $\PP(x)$ and
each patch is transformed through the learned reduction and whitening maps:
\begin{equation}
  z(u)=L^{-1}\bigl(R(u)-\muh\bigr), \qquad u\in\PP(x).
\end{equation}
We then define the squared retrieval distance
\begin{equation}
  \delta(u,m)=\|z(u)-m\|_2^2, \qquad m\in\M,
\end{equation}
and the patch-level anomaly score
\begin{equation}
  a(u)=\min_{m\in\M}\delta(u,m).
\end{equation}
We keep the squared-distance convention because exact
\ac{FAISS}~\cite{johnson2019billion} \(\ell_2\) search\footnote{This convention is
retained from the original PatchCore implementation. Squaring is monotone for
non-negative distances, so it does not change nearest-neighbour identities, but
it fixes the scale used by the downstream scoring rule.} returns squared
Euclidean distances, and all subsequent
scoring operations are defined on these squared distances.

The unreweighted image-level control is simply
\begin{equation}
  s_{\max}(x)=\max_{u\in\PP(x)} a(u).
\end{equation}
The canonical \acs{MHPC} baseline instead uses the PatchCore Eq.~(7)-style
reweighting applied to these squared distances.
Let
\begin{align}
  u^* &= \arg\max_{u\in\PP(x)} a(u), \\
  m^* &= \arg\min_{m\in\M} \delta(u^*,m),
\end{align}
and let $\mathcal{N}_b(m^*)$ denote the $b$ nearest neighbours of $m^*$ inside
$\M$, including \(m^*\) itself.
Equivalently, \(\mathcal{N}_b(m^*)\) contains \(m^*\) and its \(b-1\) nearest
distinct neighbours in the bank, so the numerator below is the \(m^*\) term in
the denominator.
Using the following numerically stabilised form, we write
\begin{align}
  \tau(x) &=
  \max_{m\in\mathcal{N}_b(m^*)}\delta(u^*,m), \label{eq:tau}\\
  w(x) &=
  1 -
  \frac{\mathrm{e}^{a(u^*)-\tau(x)}}
       {\displaystyle\sum_{m\in\mathcal{N}_b(m^*)}
        \mathrm{e}^{\delta(u^*,m)-\tau(x)}}, \label{eq:weight}\\
  s(x) &= w(x)\,a(u^*). \label{eq:score}
\end{align}
The role of the reweighting is to modulate the maximal patch anomaly by the
local support structure around its nearest memory-bank element.
Pixel-level anomaly maps are obtained by placing patch scores back on the
spatial grid and upsampling the resulting score map to image resolution.

\section{Experimental Setup}
\label{sec:experiments}

\subsection{Datasets}

\acs{MHPC} is evaluated on one public benchmark family and three selected
industrial inspection datasets.
The public benchmark is \ac{MVTecAD}~\cite{bergmann2019mvtec}.
We treat it not as a single pooled dataset but as the collection
\(\mathcal{C}_{\text{MVTec}}\) of its 15 category-specific one-class datasets.
The industrial component includes \textbf{Meniscus}, \textbf{Bottom}, and
\textbf{Lyo}\footnote{Due to non-disclosure agreements, we cannot release the
industrial datasets or identify the full commercial products from which they
are derived. We therefore report aggregate dataset statistics and use short
descriptive names.}.
Meniscus covers \ac{BFS} strip-ampoule meniscus inspection, Bottom
covers the bottom region of amber glass ampoules containing viscous liquid, and
Lyo covers lyophilised-cake inspection in transparent glass vials.
Meniscus and Bottom are evaluated on predefined \ac{ROI} crops.
This protocol reflects the intended use case of the method: evaluation should
cover a public benchmark with localisation annotations and a realistic
industrial setting with image-level labels.
Table~\ref{tab:datasets} summarises the evaluation sources, label
availability, and split-level statistics used throughout the experiments.

\begin{table}[htbp]
  \centering
  \scriptsize
  \setlength{\tabcolsep}{3pt}
  \begin{tabular}{@{}>{\raggedright\arraybackslash}p{2.55cm}>{\raggedright\arraybackslash}p{1.9cm}>{\raggedright\arraybackslash}p{1.65cm}rrr@{}}
    \toprule
    Dataset & Type & Labels & Train normal & Test & Anom.\@ eval \\
    \midrule
    \acs{MVTecAD} Bottle & Object & Image + pixel & 209 & 83 & 63 \\
    \acs{MVTecAD} Cable & Object & Image + pixel & 224 & 150 & 92 \\
    \acs{MVTecAD} Capsule & Object & Image + pixel & 219 & 132 & 109 \\
    \acs{MVTecAD} Carpet & Texture & Image + pixel & 280 & 117 & 89 \\
    \acs{MVTecAD} Grid & Texture & Image + pixel & 264 & 78 & 56 \\
    \acs{MVTecAD} Hazelnut & Object & Image + pixel & 391 & 110 & 70 \\
    \acs{MVTecAD} Leather & Texture & Image + pixel & 245 & 124 & 92 \\
    \acs{MVTecAD} Metal nut & Object & Image + pixel & 220 & 115 & 93 \\
    \acs{MVTecAD} Pill & Object & Image + pixel & 267 & 167 & 141 \\
    \acs{MVTecAD} Screw & Object & Image + pixel & 320 & 160 & 119 \\
    \acs{MVTecAD} Tile & Texture & Image + pixel & 230 & 117 & 84 \\
    \acs{MVTecAD} Toothbrush & Object & Image + pixel & 60 & 42 & 30 \\
    \acs{MVTecAD} Transistor & Object & Image + pixel & 213 & 100 & 40 \\
    \acs{MVTecAD} Wood & Texture & Image + pixel & 247 & 79 & 60 \\
    \acs{MVTecAD} Zipper & Object & Image + pixel & 240 & 151 & 119 \\
    Meniscus & \acs{BFS} strip-ampoule meniscus \acs{ROI} & Image only & 3801 & 1297 & 346 \\
    Bottom & Amber-glass-ampoule bottom \acs{ROI} & Image only & 2623 & 1301 & 645 \\
    Lyo & Lyophilised-cake vial & Image only & 3000 & 982 & 482 \\
    \bottomrule
  \end{tabular}
  \caption{Benchmark and industrial datasets used in the evaluation protocol.
  \acs{MVTecAD} is expanded by category because each category defines a separate
  one-class anomaly-detection dataset. The ``Anom.\@ eval'' column reports the
  anomalous labels entering metric computation; for \acs{MVTecAD} this follows the
  fixed image--mask preprocessing used for pixel-aware evaluation, and for the
  industrial datasets it follows the image-level test labels.}
  \label{tab:datasets}
\end{table}

\subsubsection{\acs{MVTecAD}}
\acs{MVTecAD} is the standard benchmark for unsupervised industrial \ac{AD}.
It comprises 15 categories: 10 object types (bottle, cable, capsule,
hazelnut, metal nut, pill, screw, toothbrush, transistor, zipper) and 5 texture
types (carpet, grid, leather, tile, wood).
We denote the set of categories by \(\mathcal{C}_{\text{MVTec}}\).
Each category \(c\in\mathcal{C}_{\text{MVTec}}\) defines its own training split
of only normal images and its own test split with both normal and anomalous
images, together with pixel-level binary masks.
Accordingly, all MVTec metrics are computed category-wise and reported with a
macro-average across categories rather than by pooling all categories into a
single dataset.

\subsubsection{Industrial Datasets}
Meniscus focuses on the meniscus region of liquid-filled \acs{BFS} strip ampoules
acquired from strip images.
The dominant defects are floating particles in the meniscus area. This region is
challenging because meniscus turbulence, reflections, and local appearance
variability can obscure particles and reduce the reliability of standard
blob-based or threshold-based inspection.

Bottom covers the bottom region of amber glass ampoules containing viscous
liquid. The target anomalies are particle defects, especially glass particles.
This case is challenging because heavy glass particles can exhibit limited
motion across the acquired frames, while normal bottom-region variability makes
them difficult to separate from nominal appearance. For disclosure reasons, the
Bottom input is provided as a disclosure-safe three-channel representation
derived from the three rank images rather than as the raw acquisition sequence.

Lyo covers lyophilised-cake inspection in transparent glass vials. The evaluated
anomalies include visible foreign particles, missing cake, and incorrect cake
volume. This case provides an industrial product-appearance inspection setting
with difficulty comparable to the other selected industrial datasets, because
the relevant defects may be small, localised, or expressed through product-level
appearance changes.

Annotations are available only at image level for all industrial datasets, so
they are evaluated with image-level metrics only.

\subsection{Experimental Settings}

Feature maps are extracted from the second and third residual stages of a frozen
WideResNet-50-2 backbone~\cite{zagoruyko2016wide} pre-trained on
ImageNet-1k~\cite{deng2009imagenet}.
Locally aware descriptors are built by patchifying the selected feature maps
with patch size $3$ and stride $1$, aligning them to a common spatial grid, and
aggregating them into $d_0=1024$-dimensional patch embeddings.
All images are resized to 256 pixels, centre-cropped to \(224\times224\), and
normalised with the ImageNet statistics associated with the backbone.
No training augmentation is applied for either the original batch PatchCore
baseline or the canonical streaming baseline.

The canonical \acs{MHPC} configuration uses \acs{IPCA} with retained
variance $\rho=0.99$, fixed covariance shrinkage with $\lambda=0.07$,
eigenvalue floor ratio $10^{-8}$, adaptive jittered Cholesky factorisation, and
merge-reduce $k$-center aggregation with budget $K=1000$ and local coreset size
$m_c=256$.
The retained dimension is selected as
\begin{equation}
  k=\min\left\{
  q:\frac{\sum_{i=1}^{q}\nu_i}{\sum_{i=1}^{d_0}\nu_i}\ge \rho
  \right\},
\end{equation}
where $\{\nu_i\}_{i=1}^{d_0}$ are the explained-variance components sorted in
non-increasing order.
See Appendix~\ref{app:covreg-alternatives} for covariance regularisation
details and Appendix~\ref{app:aggregation-alternatives} for the memory-bank
aggregation operators used in this study.

The \acs{FAISS}~\cite{johnson2019billion} flat \(\ell_2\) index is used for exact
nearest-neighbour search; accordingly, the reported retrieval scores are squared
Euclidean distances.
For the original PatchCore baseline, image-level scoring uses max-reduction,
$s_{\max}(x)$, to remain aligned with the original PatchCore scoring
protocol~\cite{Roth2021}.
For all \acs{MHPC} variants, image-level scoring follows the reweighted form
in Eqs.~\eqref{eq:weight}--\eqref{eq:score} with $b=9$ neighbours.
All runs use a fixed random seed.
All reported experiments are executed on \acs{CPU} only.
The public \acs{MVTecAD} experiments are run on a Linux workstation with an
Intel Core i7-9750H processor, 12 \acs{CPU} threads available to the run,
15~GiB of \acs{RAM}, 2.0~GiB of swap, Python 3.11.11, and PyTorch 2.10.0.
The industrial experiments are run on a separate Linux workstation with an
AMD Ryzen 9 9950X3D processor, 16 \acs{CPU} threads available to the run,
185~GiB of \acs{RAM}, 63~GiB of swap, Python 3.13.9, and PyTorch 2.11.0.
Both workstations include discrete NVIDIA GPUs, but GPU acceleration is not
used in the reported runs.
The industrial workstation is used because the full offline PatchCore baseline
for the industrial data exceeds the memory available on the public-benchmark
workstation.
For this reason, resource telemetry is interpreted within each result family
rather than as a cross-hardware speed comparison.

\subsection{Evaluation Metrics and Telemetry}
\label{sec:experiments:evaluation}

The primary protocol is image-level.
It is applied to each evaluation unit
\[
  \mathcal{D}_{\mathrm{ind}} =
  \{\text{Meniscus},\text{Bottom},\text{Lyo}\},
\]
\[
  d \in \mathcal{D}_{\mathrm{eval}}
  = \mathcal{C}_{\text{MVTec}} \cup
  \mathcal{D}_{\mathrm{ind}}.
\]
Let $s(x)\in\mathbb{R}$ denote the image-level anomaly score produced by the
detector and let $y(x)\in\{0,1\}$ be the corresponding image-level label.
We report \ac{AUROC} from the ranked set \(\{(s(x),y(x))\}\).
This keeps the detection comparison threshold-independent and leaves the main
emphasis on the trade-off between ranking quality, memory footprint, and
inference latency.
For \acs{MVTecAD} we report the category-wise values \(\{m_c\}_{c\in
\mathcal{C}_{\text{MVTec}}}\) together with the macro-average
\begin{equation}
  \bar m_{\text{MVTec}}=
  \frac{1}{|\mathcal{C}_{\text{MVTec}}|}
  \sum_{c\in\mathcal{C}_{\text{MVTec}}} m_c,
\label{eq:mvtec-macro}
\end{equation}
where \(|\mathcal{C}_{\text{MVTec}}|=15\).

For all evaluation units we also report resource telemetry:
\begin{equation}
  \bigl(\mathrm{RAM}_{\max},\, t_{\mathrm{fit}},\, \ell_{\mathrm{infer}}\bigr),
\end{equation}
that is, peak \acs{RAM} usage, training time, and per-image inference time.
If \(t_{\mathrm{infer}}\) is the wall-clock inference duration for an
evaluation unit with \(N_{\mathrm{test}}\) test images, then
\begin{equation}
  \ell_{\mathrm{infer}}
  = 1000\,\frac{t_{\mathrm{infer}}}{N_{\mathrm{test}}}
  \quad [\mathrm{ms/image}].
\end{equation}
The main comparison tables combine these quantities with image-level quality
because \acs{MHPC} is motivated by resource cost as well as retrieval
geometry.

Pixel-level metrics are reported only for \acs{MVTecAD}, where ground-truth masks are
available.
On \acs{MVTecAD} we report pixel-level \ac{AUROC} as a completeness metric.
For the industrial datasets, the evaluation mode is intentionally image-only,
because no spatial annotations are defined and the relevant operational question
is image-level rejection rather than pixel-level localisation.

\subsection{Industrial \acs{RAM}-Matched Subset Control}
For Meniscus we additionally study a \acs{RAM}-matched vanilla PatchCore
control.
Let
\[
  \mathcal{D}_{\mathrm{ladder}}=
  \{\text{Meniscus}\}
\]
denote the reduced-data control set.
Let
\begin{equation}
  \mathcal{N}=\{300,500,750,1000,1500,2000\}
\end{equation}
be the subset-size grid.
For each \(d\in\mathcal{D}_{\mathrm{ladder}}\) and each \(n\in\mathcal{N}\), the
control uses a fixed-seed training subset
\begin{equation}
  S_{d,n}\subset \mathcal{D}^{\mathrm{train}}_d,
  \qquad |S_{d,n}|=n,
\end{equation}
with the stated cardinality.
The reported metrics and telemetry at size \(n\) therefore correspond to the
same fixed subset ladder across all reduced-data runs.
To compare vanilla PatchCore under a \acs{RAM} budget similar to the streaming
baseline, we define the closest-\acs{RAM} operating point by
\begin{equation}
  n_d^\star=\arg\min_{n\in\mathcal{N}}
  \left|
    \mathrm{RAM}_{d,n}
    -\mathrm{RAM}^{\mathrm{stream}}_{d}
  \right|,
\label{eq:ram-match}
\end{equation}
where \(\mathrm{RAM}_{d,n}\) is the observed peak \acs{RAM} at subset size \(n\) and
\(\mathrm{RAM}^{\mathrm{stream}}_{d}\) is the peak \acs{RAM} of the streaming
\acs{MHPC} baseline on dataset \(d\).
In the results, \(n_d^\star\) identifies the closest vanilla PatchCore point
to the streaming memory budget, while the surrounding ladder is retained to
show how quality changes as the retained normal set grows.
This control is used only on Meniscus, where the central comparison concerns
image-level performance under realistic memory constraints.

\subsection{Compared Configurations}
\label{sec:experiments:configs-ablations}

The experimental design is intentionally narrow.
We compare the offline PatchCore reference with a Euclidean streaming control,
the canonical Streaming \acs{MHPC} configuration, and a
mini-batch-\(k\)-means constructor control.
Targeted studies then vary the \(k\)-center budget, evaluate the
\ac{GeoReS} variants from Section~\ref{sec:method:geores} including the
high-budget \ac{GeoReS} control, and apply the
industrial \acs{RAM}-matched subset control from Eq.~\eqref{eq:ram-match}.
The main evaluation is restricted to these contrasts so that each comparison
has a clear resource or geometric role.
Table~\ref{tab:configs} summarises the experimental contrasts.

\begin{table}[H]
  \centering
  \footnotesize
  \begin{tabular}{@{}>{\raggedright\arraybackslash}p{3.2cm}>{\raggedright\arraybackslash}p{5.2cm}>{\raggedright\arraybackslash}p{4.2cm}@{}}
    \toprule
    Configuration & Contrast in the experiment & Evaluation role \\
    \midrule
    \multicolumn{3}{@{}l}{\textit{Primary method comparisons}} \\
    Original PatchCore & Full-pool offline reference with greedy coreset selection and reference max scoring. & Non-streaming accuracy and resource baseline. \\
    \tabrowsep
    Streaming PatchCore (Euclidean) & Same streaming reducer and merge-reduce bank as \acs{MHPC}, but Euclidean retrieval. & Isolates bounded-memory construction from covariance-aware retrieval. \\
    \tabrowsep
    Streaming \acs{MHPC} & Streaming reduction, covariance whitening, merge-reduce \(k\)-center, and reweighted scoring. & Main \acs{MHPC} configuration. \\
    \tabrowsep
    Streaming \acs{MHPC} + mini-batch \(k\)-means & Same whitening pipeline with centroid summaries instead of \(k\)-center anchors. & Tests whether centroid summaries can replace coverage-oriented anchors. \\
    \tabrowsep
    \multicolumn{3}{@{}l}{\textit{\(k\)-center budget controls}} \\
    \(k\)-center large chunk & \((K,m_c)=(1000,1024)\). & Local coverage at fixed final bank size. \\
    \tabrowsep
    \(k\)-center large bank & \((K,m_c)=(5000,256)\). & Final support-set capacity at canonical local budget. \\
    \tabrowsep
    \(k\)-center high budget & \(K=5000\) and \(m_c\ge1024\) where available. & Combined local and final budget scaling. \\
    \tabrowsep
    \multicolumn{3}{@{}l}{\textit{Targeted constructor and \acs{RAM} controls}} \\
    Streaming \acs{MHPC} + \acs{GeoReS} & Two-pass residual-informed \(k\)-center, \(K=1000,\alpha=0.95\). & Residual-tail selection at the canonical bank size. \\
    \tabrowsep
    Streaming \acs{MHPC} + \acs{GeoReS}, larger bank & Two-pass residual-informed \(k\)-center, \(K=2000,\alpha=0.90\). & Residual-tail selection with more bank capacity. \\
    \tabrowsep
    Streaming \acs{MHPC} + \acs{GeoReS}, high budget & Two-pass residual-informed \(k\)-center, \(K=5000,m_c=4096,\alpha=0.90\). & High-capacity residual coverage control. \\
    \tabrowsep
    Vanilla subset sweep & Original PatchCore on fixed-seed industrial subsets \(n\in\mathcal{N}\). & \acs{RAM}-matched offline control for industrial data. \\
    \bottomrule
  \end{tabular}
  \caption{Configuration summary for the main method comparisons and targeted
  controls. The rows define the contrasts used in the reported
  accuracy--resource results.}
  \label{tab:configs}
\end{table}

The ``Streaming \acs{MHPC}'' row denotes the reference configuration
used throughout the paper.

\section{Results}
\label{sec:results}

\subsection{\acs{MVTecAD} Category-Level Benchmark Results}
\label{sec:results:main}

Table~\ref{tab:mvtec-auroc} reports the category-resolved benchmark on
\acs{MVTecAD}.
We keep the benchmark category-wise: each row is a separate one-class
problem, and the final row is the macro-average in
Eq.~\eqref{eq:mvtec-macro}.
Each cell reports the image-level value on the first line and, when available,
the corresponding pixel-level value in parentheses on the second line.
Bold values mark the best entry within each dataset row after rounding,
computed separately for the image-level and pixel-level values; ties
are bolded jointly.
Higher is better.

The offline PatchCore reference remains the strongest image-level detector on
the mean score, but the gap to Streaming \acs{MHPC} is small:
mean image \ac{AUROC} changes from \(0.991\) to \(0.989\).
The Euclidean streaming control is more clearly affected, especially on
\emph{screw}, where \ac{AUROC} drops to \(0.836\).
Because the Euclidean and Mahalanobis streaming variants share the same
bounded-memory pipeline, this gap isolates the effect of the metric geometry
after dimensionality reduction.
Replacing \(k\)-center with mini-batch \(k\)-means also reduces the mean
image-level \ac{AUROC}, which suggests that preserving support-set coverage is more
useful here than summarising the normal manifold by centroids.

At category level, the covariance-aware streaming model is not uniformly below
the offline reference: it matches or improves the rounded image-level
\ac{AUROC} on several categories, including \emph{capsule}, \emph{carpet},
\emph{grid}, \emph{pill}, \emph{wood}, and \emph{zipper}, while also tying
the near-perfect \emph{bottle}, \emph{hazelnut}, \emph{leather},
\emph{metal nut}, and \emph{toothbrush} rows after rounding.
The remaining gap is concentrated in a smaller set of harder cases, most
visibly \emph{screw}, where whitening recovers much of the Euclidean streaming
loss but does not reach the offline coreset.

The pixel-level values should be read more cautiously.
They show that the streaming geometry does not destroy localisation behaviour
on \acs{MVTecAD}: the mean pixel \ac{AUROC} changes only from \(0.980\) for the
offline reference to \(0.978\) for Streaming \acs{MHPC}.
These pixel-level values serve as public MVTec completeness rather than as the
central industrial contribution.

\begin{table}[htbp]
  \centering
  \scriptsize
  \resizebox{\textwidth}{!}{%
  \begin{tabular}{@{}l*{4}{c}@{}}
    \toprule
    Dataset &
    \shortstack{Original\\PatchCore} &
    \shortstack{Streaming\\PatchCore\\(Euclidean)} &
    \shortstack{Streaming\\Mahalanobis\\PatchCore} &
    \shortstack{Streaming\\Mahalanobis\\+ mini-batch\\$k$-means} \\
    \midrule
    Bottle & \shortstack{\textbf{1.000}\\(\textbf{0.986})} & \shortstack{\textbf{1.000}\\(0.984)} & \shortstack{\textbf{1.000}\\(0.984)} & \shortstack{\textbf{1.000}\\(0.985)} \\
    \tabrowsep
    Cable & \shortstack{\textbf{0.985}\\(\textbf{0.981})} & \shortstack{0.971\\(0.965)} & \shortstack{0.969\\(0.968)} & \shortstack{0.926\\(0.964)} \\
    \tabrowsep
    Capsule & \shortstack{0.990\\(\textbf{0.989})} & \shortstack{0.988\\(0.983)} & \shortstack{\textbf{0.994}\\(0.987)} & \shortstack{0.971\\(\textbf{0.989})} \\
    \tabrowsep
    Carpet & \shortstack{0.980\\(\textbf{0.988})} & \shortstack{0.970\\(0.986)} & \shortstack{\textbf{0.985}\\(0.987)} & \shortstack{0.983\\(\textbf{0.988})} \\
    \tabrowsep
    Grid & \shortstack{0.997\\(0.984)} & \shortstack{0.995\\(0.976)} & \shortstack{\textbf{1.000}\\(\textbf{0.985})} & \shortstack{0.998\\(0.984)} \\
    \tabrowsep
    Hazelnut & \shortstack{\textbf{1.000}\\(\textbf{0.986})} & \shortstack{\textbf{1.000}\\(0.984)} & \shortstack{\textbf{1.000}\\(0.984)} & \shortstack{\textbf{1.000}\\(0.985)} \\
    \tabrowsep
    Leather & \shortstack{\textbf{1.000}\\(0.992)} & \shortstack{\textbf{1.000}\\(\textbf{0.993})} & \shortstack{\textbf{1.000}\\(0.992)} & \shortstack{\textbf{1.000}\\(0.992)} \\
    \tabrowsep
    Metal nut & \shortstack{\textbf{1.000}\\(0.985)} & \shortstack{0.997\\(0.970)} & \shortstack{\textbf{1.000}\\(\textbf{0.988})} & \shortstack{0.999\\(\textbf{0.988})} \\
    \tabrowsep
    Pill & \shortstack{0.961\\(0.976)} & \shortstack{0.955\\(0.954)} & \shortstack{\textbf{0.966}\\(\textbf{0.982})} & \shortstack{0.958\\(0.980)} \\
    \tabrowsep
    Screw & \shortstack{\textbf{0.984}\\(\textbf{0.994})} & \shortstack{0.836\\(0.867)} & \shortstack{0.937\\(0.982)} & \shortstack{0.907\\(0.986)} \\
    \tabrowsep
    Tile & \shortstack{0.990\\(0.960)} & \shortstack{\textbf{0.992}\\(\textbf{0.962})} & \shortstack{0.991\\(0.957)} & \shortstack{0.989\\(0.956)} \\
    \tabrowsep
    Toothbrush & \shortstack{\textbf{1.000}\\(\textbf{0.987})} & \shortstack{\textbf{1.000}\\(0.984)} & \shortstack{\textbf{1.000}\\(\textbf{0.987})} & \shortstack{\textbf{1.000}\\(\textbf{0.987})} \\
    \tabrowsep
    Transistor & \shortstack{\textbf{1.000}\\(\textbf{0.955})} & \shortstack{0.999\\(0.932)} & \shortstack{0.999\\(0.943)} & \shortstack{0.985\\(0.927)} \\
    \tabrowsep
    Wood & \shortstack{0.990\\(\textbf{0.950})} & \shortstack{0.990\\(0.945)} & \shortstack{\textbf{0.995}\\(\textbf{0.950})} & \shortstack{\textbf{0.995}\\(0.948)} \\
    \tabrowsep
    Zipper & \shortstack{0.994\\(0.988)} & \shortstack{0.996\\(0.986)} & \shortstack{\textbf{0.997}\\(\textbf{0.989})} & \shortstack{0.993\\(\textbf{0.989})} \\
    \tabrowsep
    Mean & \shortstack{\textbf{0.991}\\(\textbf{0.980})} & \shortstack{0.979\\(0.965)} & \shortstack{0.989\\(0.978)} & \shortstack{0.980\\(0.977)} \\
    \bottomrule
  \end{tabular}}
  \caption{\acs{MVTecAD} \acs{AUROC}. Each cell reports image-level
  \acs{AUROC} and, in parentheses, pixel-level \acs{AUROC}.}
  \label{tab:mvtec-auroc}
\end{table}

\subsection{Industrial Image-Level Results}

Table~\ref{tab:industrial-auroc} mirrors the MVTec category-level table for
the industrial protocol.
Each row is one selected industrial inspection dataset.
Unlike the MVTec table, this table reports image-level \ac{AUROC} only, because
the industrial data used here do not provide pixel-level masks.

Across the selected industrial datasets, the offline PatchCore reference is
still the strongest image-level detector among the four main configurations,
with mean \ac{AUROC} \(0.996\).
Within the bounded-memory family, the Mahalanobis variant improves over the
Euclidean streaming control in mean \ac{AUROC} (\(0.986\) versus \(0.981\)),
while the mini-batch \(k\)-means constructor is weaker on this selected
industrial set.
This supports the central comparison: the streaming-compatible construction is
most useful when paired with a memory bank that preserves support-set coverage
and a metric geometry that reflects the covariance of the normal features.

\begin{table}[htbp]
  \centering
  \scriptsize
  \resizebox{\textwidth}{!}{%
  \begin{tabular}{@{}l*{4}{c}@{}}
    \toprule
    Dataset &
    \shortstack{Original\\PatchCore} &
    \shortstack{Streaming\\PatchCore\\(Euclidean)} &
    \shortstack{Streaming\\Mahalanobis\\PatchCore} &
    \shortstack{Streaming\\Mahalanobis\\+ mini-batch\\$k$-means} \\
    \midrule
    Meniscus & \textbf{0.992} & 0.962 & 0.974 & 0.901 \\
    \tabrowsep
    Bottom & \textbf{0.998} & 0.995 & 0.995 & 0.989 \\
    \tabrowsep
    Lyo & \textbf{0.998} & 0.985 & 0.988 & 0.982 \\
    \tabrowsep
    Mean & \textbf{0.996} & 0.981 & 0.986 & 0.957 \\
    \bottomrule
  \end{tabular}}
  \caption{Industrial image-level \acs{AUROC} on the selected industrial
  datasets. Pixel-level metrics are not reported for the industrial datasets
  because the protocol uses image-level labels only.}
  \label{tab:industrial-auroc}
\end{table}

Figures~\ref{fig:industrial-meniscus-baseline-qualitative},
\ref{fig:industrial-bottom-baseline-qualitative}, and
\ref{fig:industrial-lyo-baseline-qualitative} show qualitative examples from
the canonical streaming \acs{MHPC} baseline on the three selected industrial
datasets.
For each sample, the grids show the input \acs{ROI}, the anomaly-score heatmap
overlaid on the \acs{ROI}, and the predicted segmentation overlay.
These maps provide visual context for the selected industrial cases, while the
quantitative protocol remains image-level.

\begin{figure}[htbp]
  \centering
  \includegraphics[width=\textwidth]{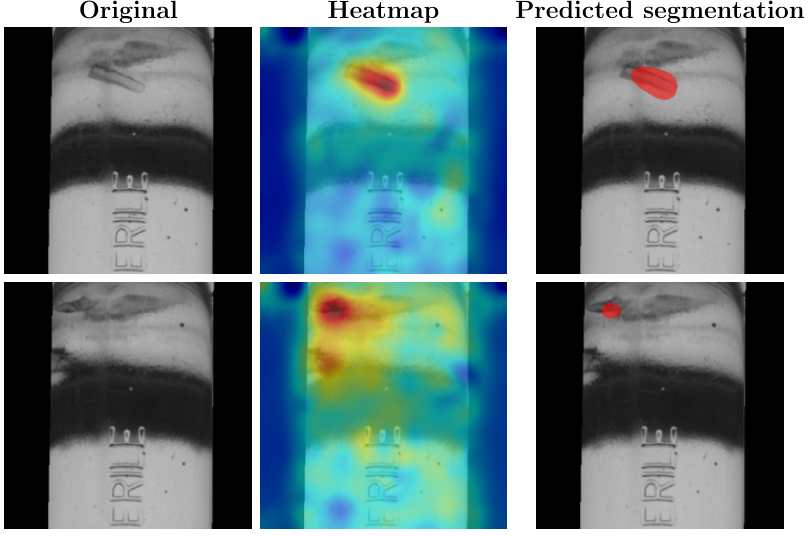}
  \caption{Meniscus qualitative examples from the canonical streaming
  \acs{MHPC} baseline. Each row shows one sample as the original image,
  anomaly-score heatmap overlay, and predicted segmentation overlay.}
  \label{fig:industrial-meniscus-baseline-qualitative}
\end{figure}

\begin{figure}[htbp]
  \centering
  \includegraphics[width=\textwidth]{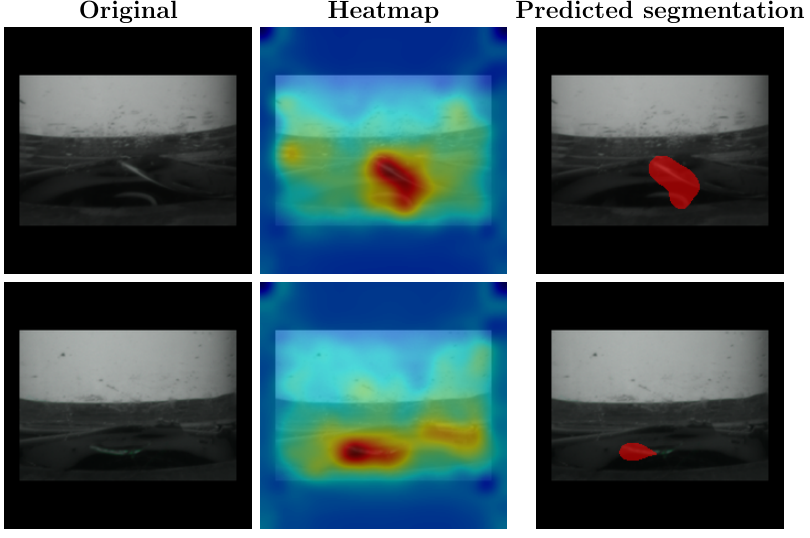}
  \caption{Bottom qualitative examples from the canonical streaming
  \acs{MHPC} baseline. Each row shows one sample as the original image,
  anomaly-score heatmap overlay, and predicted segmentation overlay.}
  \label{fig:industrial-bottom-baseline-qualitative}
\end{figure}

\begin{figure}[htbp]
  \centering
  \includegraphics[width=\textwidth]{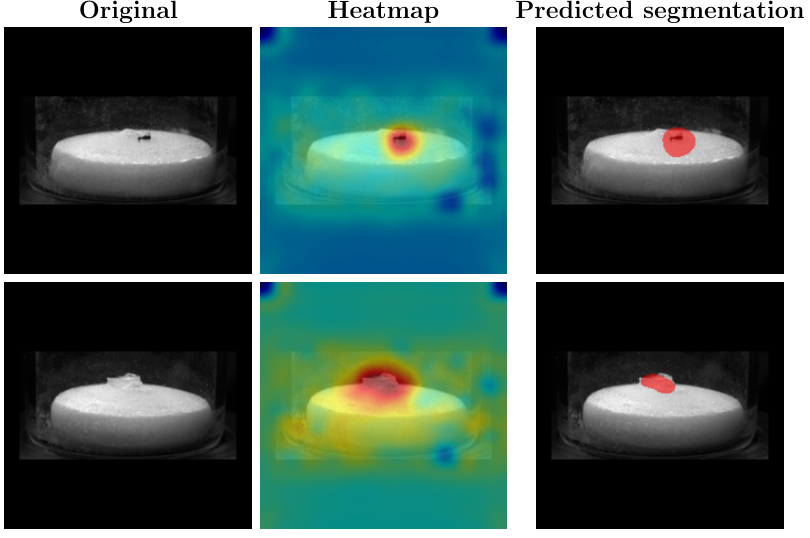}
  \caption{Lyo qualitative examples from the canonical streaming \acs{MHPC}
  baseline. Each row shows one sample as the original image, anomaly-score
  heatmap overlay, and predicted segmentation overlay.}
  \label{fig:industrial-lyo-baseline-qualitative}
\end{figure}

Table~\ref{tab:industrial-subset-ladder} gives the industrial-only
reduced-data control from Eq.~\eqref{eq:ram-match}.
It compares original PatchCore, the canonical streaming baseline, and the
\ac{GeoReS} \(K=2000,\alpha=0.90\) trade-off on the same fixed-seed Meniscus
subset ladder.
Figure~\ref{fig:industrial-ram-ladder} visualises the corresponding \acs{RAM} growth
for original PatchCore and the \ac{GeoReS} trade-off.
Original PatchCore improves as more normal images are retained, from
\ac{AUROC} \(0.958\) at \(n=300\) to \(0.988\) at \(n=2000\), but its peak
\acs{RAM} grows from \(8.34\) to \(20.46\)~GB.
The \ac{GeoReS} \(K=2000\) configuration remains in the \(8.24\)--\(9.22\)~GB
range across the same ladder and increases from \ac{AUROC} \(0.952\) to
\(0.984\).
Thus, at a near-\(9\)~GB memory scale, bounded-memory training
can use the larger normal subset, whereas original PatchCore must stay near
the smallest subset sizes.
Using Eq.~\eqref{eq:ram-match}, the closest vanilla PatchCore point to the
canonical streaming \acs{MHPC} memory budget is \(n=300\), with \ac{AUROC}
\(0.958\) at \(8.34\)~GB.
Within the adjacent neighbourhood around \(9\)~GB, original PatchCore reaches
\ac{AUROC} \(0.979\) at \(n=750\) and \(9.02\)~GB, while the \ac{GeoReS}
\(n=2000\) row reaches \ac{AUROC} \(0.984\) at \(8.24\)~GB.

\begin{table}[htbp]
  \centering
  \scriptsize
  \resizebox{\textwidth}{!}{%
  \begin{tabular}{@{}llcccc@{}}
    \toprule
    \shortstack{Normal train\\images} & Configuration & \acs{AUROC} &
    \shortstack{Max \acs{RAM}\\(GB)} & \shortstack{Fit time\\(s)} &
    \shortstack{Inference latency\\(ms/image)} \\
    \midrule
    300 & Original PatchCore & \textbf{0.958} & \textbf{8.34} & 143.1 & 839.5 \\
    \tabrowsep
    300 & Streaming Mahalanobis & 0.938 & 8.87 & \textbf{121.8} & \textbf{61.9} \\
    \tabrowsep
    300 & \acs{GeoReS} \(K=2000,\alpha=0.90\) & 0.952 & 9.13 & 654.5 & 84.4 \\
    \midrule
    500 & Original PatchCore & \textbf{0.974} & 9.17 & 377.2 & 1389.4 \\
    \tabrowsep
    500 & Streaming Mahalanobis & 0.945 & \textbf{8.00} & \textbf{202.4} & \textbf{55.8} \\
    \tabrowsep
    500 & \acs{GeoReS} \(K=2000,\alpha=0.90\) & 0.967 & 8.49 & 1060.5 & 74.6 \\
    \midrule
    750 & Original PatchCore & \textbf{0.979} & 9.02 & 863.4 & 2087.2 \\
    \tabrowsep
    750 & Streaming Mahalanobis & 0.951 & \textbf{8.99} & \textbf{304.1} & \textbf{65.2} \\
    \tabrowsep
    750 & \acs{GeoReS} \(K=2000,\alpha=0.90\) & 0.973 & 9.05 & 1682.2 & 84.9 \\
    \midrule
    1000 & Original PatchCore & \textbf{0.982} & 11.13 & 1406.7 & 2861.9 \\
    \tabrowsep
    1000 & Streaming Mahalanobis & 0.963 & \textbf{8.01} & \textbf{408.5} & \textbf{65.9} \\
    \tabrowsep
    1000 & \acs{GeoReS} \(K=2000,\alpha=0.90\) & 0.977 & 9.22 & 2161.0 & 85.3 \\
    \midrule
    1500 & Original PatchCore & \textbf{0.985} & 15.82 & 3376.1 & 3961.4 \\
    \tabrowsep
    1500 & Streaming Mahalanobis & 0.959 & \textbf{7.95} & \textbf{607.3} & \textbf{65.3} \\
    \tabrowsep
    1500 & \acs{GeoReS} \(K=2000,\alpha=0.90\) & 0.981 & 8.59 & 3135.0 & 85.6 \\
    \midrule
    2000 & Original PatchCore & \textbf{0.988} & 20.46 & 6108.9 & 5505.9 \\
    \tabrowsep
    2000 & Streaming Mahalanobis & 0.971 & 8.81 & \textbf{821.5} & \textbf{66.0} \\
    \tabrowsep
    2000 & \acs{GeoReS} \(K=2000,\alpha=0.90\) & 0.984 & \textbf{8.24} & 4469.1 & 85.5 \\
    \bottomrule
  \end{tabular}}
  \caption{Industrial Meniscus reduced-data comparison. Each block uses the same
  fixed-seed normal-training subset and compares original PatchCore with the
  canonical streaming baseline and the selected streaming memory--accuracy tradeoff
  used in the \acs{RAM}-growth plot. Peak \acs{RAM} is reported in GB, fit time in seconds,
  and inference latency in milliseconds per image; lower is better for
  telemetry.}
  \label{tab:industrial-subset-ladder}
\end{table}

\begin{figure}[htbp]
  \centering
  \includegraphics[width=0.72\textwidth]{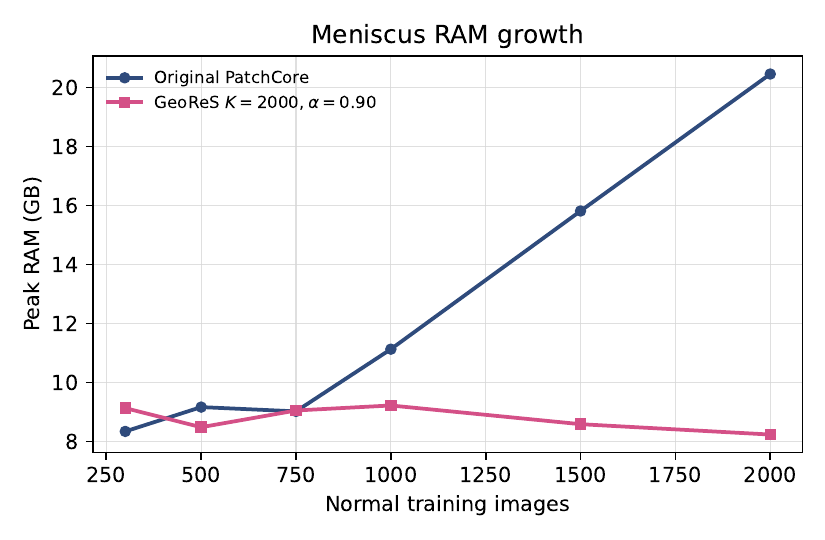}
  \caption{Meniscus \acs{RAM} growth for original PatchCore and the selected
  \acs{GeoReS} streaming trade-off. Original PatchCore improves by retaining
  more normal images, but its resident memory grows with the retained subset,
  whereas the bounded-memory construction stays near a fixed \acs{RAM} band.}
  \label{fig:industrial-ram-ladder}
\end{figure}

\clearpage

\subsection{Targeted Ablations and Budget Studies}
\label{sec:results:ablation}

\paragraph{Budget controls}
Table~\ref{tab:mvtec-kcenter-budget} isolates the two \(k\)-center budget
parameters on \acs{MVTecAD}: the final memory-bank size \(K\) and the local
merge-reduce coreset size \(m_c\).
Increasing the local chunk budget from \(m_c=256\) to \(m_c=1024\) gives a
small improvement in mean image and pixel \ac{AUROC}, but the gain is minor
relative to the additional training time.
Increasing the final bank from \(K=1000\) to \(K=5000\) does not improve the
MVTec image-level mean under the evaluated protocol, even though the joint
high-budget setting improves the reported pixel summaries.
The higher-budget configurations are still informative:
\((K,m_c)=(1000,1024)\) reaches
the offline PatchCore mean image \ac{AUROC} after rounding while keeping peak
\acs{RAM} at \(2.56\)~GB rather than the offline reference's \(5.41\)~GB.
For the paper's main baseline, however, the canonical \((K,m_c)=(1000,256)\)
setting remains the intended lightweight operating point, because it gives most
of this quality with the lowest training cost and fastest inference among the
covariance-aware \(k\)-center variants.

\begin{table}[htbp]
  \centering
  \footnotesize
  \resizebox{\textwidth}{!}{%
  \begin{tabular}{@{}l*{5}{c}|c|@{}}
    \toprule
    Measure &
    \shortstack{Streaming\\Mahalanobis\\$K=1000,m_c=256$} &
    \shortstack{Streaming\\Mahalanobis\\$K=1000,m_c=1024$} &
    \shortstack{Streaming\\Mahalanobis\\$K=5000,m_c=256$} &
    \shortstack{Streaming\\Mahalanobis\\$K=5000,m_c=1024$} &
    \shortstack{Streaming\\Mahalanobis\\$K=5000,m_c=4096$} &
    \shortstack{Original\\PatchCore} \\
    \midrule
    Image \acs{AUROC} & 0.989 & \textbf{0.991} & 0.989 & 0.989 & 0.989 & 0.991 \\
    \tabrowsep
    Pixel \acs{AUROC} & 0.978 & \textbf{0.979} & 0.978 & 0.978 & \textbf{0.979} & 0.980 \\
    \tabrowsep
    Mean training time (s) & \textbf{57.913} & 162.627 & 231.418 & 171.598 & 1330.624 & 12.444 \\
    \tabrowsep
    Mean \acs{RAM} (GB) & 2.430 & \textbf{2.424} & 2.749 & 2.446 & 3.585 & 4.125 \\
    \tabrowsep
    Max \acs{RAM} (GB) & 2.775 & \textbf{2.556} & 3.083 & 2.761 & 3.853 & 5.412 \\
    \tabrowsep
    Per-image inference time (ms/image) & \textbf{42.610} & 46.437 & 87.706 & 63.559 & 135.963 & 467.092 \\
    \bottomrule
  \end{tabular}}
  \caption{\acs{MVTecAD} \(k\)-center budget study. \acs{AUROC} is macro-averaged over
  the 15 category datasets. Mean training time is reported
  in seconds, \acs{RAM} in GB, and per-image inference time as an image-weighted
  mean. The vertically separated Original PatchCore column is an offline
  reference and is excluded from boldface selection.}
  \label{tab:mvtec-kcenter-budget}
\end{table}

Table~\ref{tab:industrial-kcenter-budget} reports the same \(k\)-center budget
study for the industrial datasets.
The table has the same role as Table~\ref{tab:mvtec-kcenter-budget}, but it
omits pixel-level columns because the industrial protocol is image-level only.
On the selected industrial datasets, the joint high-budget setting
\((K,m_c)=(5000,4096)\) gives the strongest \(k\)-center configuration,
raising mean image \ac{AUROC} from \(0.986\) to \(0.993\), while keeping peak
\acs{RAM} at \(8.55\)~GB rather than the offline PatchCore \(37.43\)~GB footprint.
This reinforces the intended interpretation of the budget ladder: the canonical
configuration is the lightweight default, whereas larger summaries expose the
accuracy that can be recovered when additional training time is acceptable.

\begin{table}[htbp]
  \centering
  \footnotesize
  \resizebox{\textwidth}{!}{%
  \begin{tabular}{@{}l*{5}{c}|c|@{}}
    \toprule
    Measure &
    \shortstack{Streaming\\Mahalanobis\\$K=1000,m_c=256$} &
    \shortstack{Streaming\\Mahalanobis\\$K=1000,m_c=1024$} &
    \shortstack{Streaming\\Mahalanobis\\$K=5000,m_c=256$} &
    \shortstack{Streaming\\Mahalanobis\\$K=5000,m_c=1024$} &
    \shortstack{Streaming\\Mahalanobis\\$K=5000,m_c=4096$} &
    \shortstack{Original\\PatchCore} \\
    \midrule
    Image \acs{AUROC} & 0.986 & 0.987 & 0.985 & 0.990 & \textbf{0.993} & 0.996 \\
    \tabrowsep
    Mean training time (s) & \textbf{1078.977} & 2697.972 & 1382.193 & 2652.466 & 12650.231 & 14837.300 \\
    \tabrowsep
    Mean \acs{RAM} (GB) & \textbf{8.001} & 8.283 & 8.595 & 8.395 & 8.064 & 31.462 \\
    \tabrowsep
    Max \acs{RAM} (GB) & \textbf{8.227} & 8.826 & 9.612 & 9.309 & 8.553 & 37.430 \\
    \tabrowsep
    Per-image inference time (ms/image) & \textbf{60.702} & 68.484 & 71.170 & 102.706 & 108.006 & 8207.638 \\
    \bottomrule
  \end{tabular}}
  \caption{Industrial \(k\)-center budget study on the selected industrial
  datasets. \acs{AUROC} is macro-averaged over the selected datasets. Mean training
  time is reported in seconds, \acs{RAM} in GB, and per-image inference
  time as an image-weighted mean. The vertically separated Original PatchCore
  column is an offline reference and is excluded from boldface selection.}
  \label{tab:industrial-kcenter-budget}
\end{table}

\paragraph{Residual-informed two-pass ablation}
Table~\ref{tab:mvtec-geores-ablation} reports the residual-informed
two-pass variants defined in Section~\ref{sec:method:geores}.
The comparison is anchored by the canonical Streaming \acs{MHPC}
configuration, so the three \ac{GeoReS} configurations should be read as targeted
ablations of the paper's baseline rather than as a separate offline comparison.
At the same final bank size as the canonical baseline, \ac{GeoReS} improves
mean image \ac{AUROC} from \(0.989\) to \(0.992\), exceeding the offline
PatchCore reference after rounding while remaining a two-pass ablation rather
than the canonical operating point.
The larger residual-informed variants mainly affect the MVTec-only pixel
summaries; the \(K=5000,m_c=4096\) setting gives the strongest reported pixel
scores but not the strongest image-level macro result.
This gain is modest and should be read together with the cost profile:
because \ac{GeoReS} must first estimate residual-heavy regions and then build
the final bank, mean training time increases from roughly \(58\)~s for the
canonical streaming baseline to about \(368\)--\(467\)~s for the lower-budget
\ac{GeoReS} variants and \(3646\)~s for the high-budget variant.
Image-weighted per-image inference time also increases, from about
\(42.6\)~ms/image to \(46.4\)--\(55.9\)~ms/image, and to \(151.7\)~ms/image
for the high-budget variant, so the extra training and inference costs remain
the main trade-offs.

\begin{table}[htbp]
  \centering
  \scriptsize
  \resizebox{\textwidth}{!}{%
  \begin{tabular}{@{}l*{4}{c}|c|@{}}
    \toprule
    Measure &
    \shortstack{Streaming\\Mahalanobis\\PatchCore} &
    \shortstack{Streaming\\Mahalanobis\\+ \acs{GeoReS}\\$K=1000,\alpha=0.95$} &
    \shortstack{Streaming\\Mahalanobis\\+ \acs{GeoReS}\\$K=2000,\alpha=0.90$} &
    \shortstack{Streaming\\Mahalanobis\\+ \acs{GeoReS}\\$K=5000,\alpha=0.90$} &
    \shortstack{Original\\PatchCore} \\
    \midrule
    Image \acs{AUROC} & 0.989 & \textbf{0.992} & 0.991 & 0.990 & 0.991 \\
    \tabrowsep
    Pixel \acs{AUROC} & 0.978 & 0.978 & 0.978 & \textbf{0.980} & 0.980 \\
    \tabrowsep
    Mean training time (s) & \textbf{57.913} & 467.414 & 368.037 & 3646.298 & 12.444 \\
    \tabrowsep
    Per-image inference time (ms/image) & \textbf{42.610} & 46.367 & 55.881 & 151.738 & 467.092 \\
    \tabrowsep
    Max \acs{RAM} (GB) & 2.775 & 2.712 & \textbf{2.622} & 4.087 & 5.412 \\
    \bottomrule
  \end{tabular}}
  \caption{\acs{MVTecAD} ablation of two-pass geometric residual selection.
  \acs{AUROC} is macro-averaged over the 15 category datasets.
  Pixel-level \acs{AUROC} is included only for \acs{MVTecAD} localisation completeness.
  \acs{RAM} is reported as peak resident memory in GB, mean training time in seconds,
  and per-image inference time as an image-weighted mean. Lower is better for
  \acs{RAM} and time. The vertically separated Original PatchCore column is an
  offline reference and is excluded from boldface selection.}
  \label{tab:mvtec-geores-ablation}
\end{table}

Table~\ref{tab:industrial-geores-ablation} gives the corresponding industrial
\ac{GeoReS} ablation.
Across the selected industrial datasets, the \(K=2000,\alpha=0.90\)
residual-informed variant improves mean image \ac{AUROC} from \(0.986\) to
\(0.991\), while also increasing training and inference time.
The high-budget \(K=5000,m_c=4096\) configuration reaches the strongest mean
\ac{AUROC}, \(0.994\), but increases mean training time to about
\(24945\)~s.
Together with the \acs{MVTecAD} result, where \ac{GeoReS} changes the macro image
\ac{AUROC} only from \(0.989\) to \(0.992\), this suggests that residual-tail
selection is most useful when the normal support contains variable modes or
locally under-represented nominal appearance regions.
In the current industrial evidence, this interpretation is consistent with the
three selected datasets: Meniscus contains meniscus and fluid variability,
Bottom contains bottom-region variability in viscous-liquid ampoules, and Lyo
contains product-appearance anomalies such as visible foreign particles, missing
cake, and incorrect cake volume. Across these cases, the relevant defect
evidence can be small, localised, or subtle, while the corresponding nominal
appearance modes may be sparsely sampled; this makes additional
residual-guided coverage useful at image level.

\begin{table}[htbp]
  \centering
  \scriptsize
  \resizebox{\textwidth}{!}{%
  \begin{tabular}{@{}l*{4}{c}|c|@{}}
    \toprule
    Measure &
    \shortstack{Streaming\\Mahalanobis\\PatchCore} &
    \shortstack{Streaming\\Mahalanobis\\+ \acs{GeoReS}\\$K=1000,\alpha=0.95$} &
    \shortstack{Streaming\\Mahalanobis\\+ \acs{GeoReS}\\$K=2000,\alpha=0.90$} &
    \shortstack{Streaming\\Mahalanobis\\+ \acs{GeoReS}\\$K=5000,\alpha=0.90$} &
    \shortstack{Original\\PatchCore} \\
    \midrule
    Image \acs{AUROC} & 0.986 & 0.988 & 0.991 & \textbf{0.994} & 0.996 \\
    \tabrowsep
    Mean training time (s) & \textbf{1078.977} & 7027.116 & 5332.320 & 24945.293 & 14837.300 \\
    \tabrowsep
    Per-image inference time (ms/image) & \textbf{60.702} & 68.344 & 78.489 & 106.992 & 8207.638 \\
    \tabrowsep
    Max \acs{RAM} (GB) & \textbf{8.227} & 9.569 & 9.332 & 9.383 & 37.430 \\
    \bottomrule
  \end{tabular}}
  \caption{Industrial ablation of two-pass geometric residual selection on the
  selected industrial datasets. \acs{AUROC} is macro-averaged over the selected
  datasets. \acs{RAM} is reported as peak resident memory in GB, mean
  training time in seconds, and per-image inference time as an image-weighted
  mean. Lower is better for \acs{RAM} and time. The vertically separated Original
  PatchCore column is an offline reference and is excluded from boldface
  selection.}
  \label{tab:industrial-geores-ablation}
\end{table}

\clearpage

\paragraph{Dataset-resolved ablation views}
The category-resolved \acs{MVTecAD} \ac{AUROC} views are reported in
Tables~\ref{tab:mvtec-kcenter-budget-auroc}
and~\ref{tab:mvtec-geores-ablation-auroc}.
The corresponding industrial dataset-resolved views are reported in
Tables~\ref{tab:industrial-kcenter-budget-auroc}
and~\ref{tab:industrial-geores-ablation-auroc}.
The \acs{MVTecAD} tables report image-level \ac{AUROC} with pixel-level
\ac{AUROC} in parentheses, because pixel masks are available for that
benchmark.
The industrial tables remain image-level only and include Original PatchCore
as a vertically separated offline reference; boldface still selects among the
streaming variants, so the ablation comparison remains anchored by the
canonical streaming baseline.

These dataset-resolved views refine the macro story rather than changing it.
For the \(k\)-center budget study, larger local summaries help some difficult
categories, for example increasing image \ac{AUROC} on \emph{screw} from
\(0.937\) to \(0.964\) when \(m_c\) rises from \(256\) to \(1024\) at
\(K=1000\).
The effect is not monotone across categories, because the high-budget setting
also reduces some rows, such as \emph{pill}.
This is not a contradiction in the budget study: these settings are not nested
models. Changing \(K\) and \(m_c\) changes the streaming coreset construction
itself, so a larger bank need not preserve the same image-level ranking,
especially on near-ceiling MVTec categories where additional normal coverage
can also reduce anomaly distances.
\ac{GeoReS} shows a similar pattern: the larger residual-informed variant
improves \emph{screw} from the canonical baseline \(0.937\) to \(0.981\), but
the macro-average changes only from \(0.989\) to \(0.992\) at best.
Thus the category tables support a conservative interpretation: extra coverage
or residual-tail selection can improve particular categories, but the canonical
streaming memory bank is already close to saturated on the aggregate MVTec
benchmark.
The industrial rows show a stronger budget and residual-tail signal.
The \(k\)-center budget sweep raises mean image \ac{AUROC} from \(0.986\) for
the canonical streaming baseline to \(0.993\) at \(K=5000,m_c=4096\), and the
\ac{GeoReS} sweep reaches \(0.994\) at the largest evaluated setting.
Most of that visible room comes from Meniscus and Lyo; Bottom is already near
saturated for the canonical streaming baseline.

\begin{table}[htbp]
  \centering
  \scriptsize
  \resizebox{\textwidth}{!}{%
  \begin{tabular}{@{}l*{5}{c}|c|@{}}
    \toprule
    Dataset &
    \shortstack{Streaming\\Mahalanobis\\PatchCore\\$K=1000,m_c=256$} &
    \shortstack{Streaming\\Mahalanobis\\$K=1000,m_c=1024$} &
    \shortstack{Streaming\\Mahalanobis\\$K=5000,m_c=256$} &
    \shortstack{Streaming\\Mahalanobis\\$K=5000,m_c=1024$} &
    \shortstack{Streaming\\Mahalanobis\\$K=5000,m_c=4096$} &
    \shortstack{Original\\PatchCore} \\
    \midrule
    Bottle & \shortstack{\textbf{1.000}\\(0.984)} & \shortstack{\textbf{1.000}\\(0.984)} & \shortstack{\textbf{1.000}\\(0.984)} & \shortstack{0.999\\(\textbf{0.985})} & \shortstack{\textbf{1.000}\\(\textbf{0.985})} & \shortstack{1.000\\(0.986)} \\
    \tabrowsep
    Cable & \shortstack{0.969\\(0.968)} & \shortstack{\textbf{0.970}\\(0.968)} & \shortstack{0.967\\(0.969)} & \shortstack{0.969\\(0.969)} & \shortstack{0.961\\(\textbf{0.972})} & \shortstack{0.985\\(0.981)} \\
    \tabrowsep
    Capsule & \shortstack{\textbf{0.994}\\(0.987)} & \shortstack{\textbf{0.994}\\(\textbf{0.989})} & \shortstack{\textbf{0.994}\\(0.987)} & \shortstack{0.992\\(\textbf{0.989})} & \shortstack{0.990\\(\textbf{0.989})} & \shortstack{0.990\\(0.989)} \\
    \tabrowsep
    Carpet & \shortstack{\textbf{0.985}\\(0.987)} & \shortstack{0.983\\(0.987)} & \shortstack{0.984\\(0.987)} & \shortstack{0.984\\(\textbf{0.988})} & \shortstack{0.983\\(\textbf{0.988})} & \shortstack{0.980\\(0.988)} \\
    \tabrowsep
    Grid & \shortstack{\textbf{1.000}\\(0.985)} & \shortstack{\textbf{1.000}\\(\textbf{0.986})} & \shortstack{\textbf{1.000}\\(0.984)} & \shortstack{0.999\\(0.984)} & \shortstack{\textbf{1.000}\\(\textbf{0.986})} & \shortstack{0.997\\(0.984)} \\
    \tabrowsep
    Hazelnut & \shortstack{\textbf{1.000}\\(0.984)} & \shortstack{\textbf{1.000}\\(0.985)} & \shortstack{\textbf{1.000}\\(0.984)} & \shortstack{\textbf{1.000}\\(0.984)} & \shortstack{\textbf{1.000}\\(\textbf{0.986})} & \shortstack{1.000\\(0.986)} \\
    \tabrowsep
    Leather & \shortstack{\textbf{1.000}\\(\textbf{0.992})} & \shortstack{\textbf{1.000}\\(\textbf{0.992})} & \shortstack{\textbf{1.000}\\(\textbf{0.992})} & \shortstack{\textbf{1.000}\\(\textbf{0.992})} & \shortstack{\textbf{1.000}\\(\textbf{0.992})} & \shortstack{1.000\\(0.992)} \\
    \tabrowsep
    Metal nut & \shortstack{\textbf{1.000}\\(\textbf{0.988})} & \shortstack{\textbf{1.000}\\(\textbf{0.988})} & \shortstack{\textbf{1.000}\\(\textbf{0.988})} & \shortstack{\textbf{1.000}\\(\textbf{0.988})} & \shortstack{\textbf{1.000}\\(\textbf{0.988})} & \shortstack{1.000\\(0.985)} \\
    \tabrowsep
    Pill & \shortstack{0.966\\(\textbf{0.982})} & \shortstack{\textbf{0.981}\\(\textbf{0.982})} & \shortstack{0.966\\(\textbf{0.982})} & \shortstack{0.951\\(0.976)} & \shortstack{0.963\\(0.977)} & \shortstack{0.961\\(0.976)} \\
    \tabrowsep
    Screw & \shortstack{0.937\\(0.982)} & \shortstack{0.964\\(0.989)} & \shortstack{0.937\\(0.982)} & \shortstack{0.966\\(0.990)} & \shortstack{\textbf{0.974}\\(\textbf{0.993})} & \shortstack{0.984\\(0.994)} \\
    \tabrowsep
    Tile & \shortstack{0.991\\(\textbf{0.957})} & \shortstack{\textbf{0.992}\\(0.956)} & \shortstack{0.991\\(\textbf{0.957})} & \shortstack{0.991\\(\textbf{0.957})} & \shortstack{0.990\\(0.956)} & \shortstack{0.990\\(0.960)} \\
    \tabrowsep
    Toothbrush & \shortstack{\textbf{1.000}\\(0.987)} & \shortstack{\textbf{1.000}\\(0.987)} & \shortstack{\textbf{1.000}\\(0.987)} & \shortstack{\textbf{1.000}\\(0.987)} & \shortstack{\textbf{1.000}\\(\textbf{0.988})} & \shortstack{1.000\\(0.987)} \\
    \tabrowsep
    Transistor & \shortstack{\textbf{0.999}\\(0.943)} & \shortstack{0.998\\(\textbf{0.950})} & \shortstack{\textbf{0.999}\\(0.943)} & \shortstack{0.996\\(0.939)} & \shortstack{0.994\\(\textbf{0.950})} & \shortstack{1.000\\(0.955)} \\
    \tabrowsep
    Wood & \shortstack{\textbf{0.995}\\(\textbf{0.950})} & \shortstack{0.993\\(\textbf{0.950})} & \shortstack{\textbf{0.995}\\(\textbf{0.950})} & \shortstack{0.993\\(\textbf{0.950})} & \shortstack{0.991\\(0.949)} & \shortstack{0.990\\(0.950)} \\
    \tabrowsep
    Zipper & \shortstack{0.997\\(\textbf{0.989})} & \shortstack{\textbf{0.998}\\(\textbf{0.989})} & \shortstack{0.997\\(\textbf{0.989})} & \shortstack{0.997\\(\textbf{0.989})} & \shortstack{0.995\\(\textbf{0.989})} & \shortstack{0.994\\(0.988)} \\
    \tabrowsep
    Mean & \shortstack{0.989\\(0.978)} & \shortstack{\textbf{0.991}\\(\textbf{0.979})} & \shortstack{0.989\\(0.978)} & \shortstack{0.989\\(0.978)} & \shortstack{0.989\\(\textbf{0.979})} & \shortstack{0.991\\(0.980)} \\
    \bottomrule
  \end{tabular}}
  \caption{\acs{MVTecAD} category-wise \(k\)-center budget ablation in \acs{AUROC}. The comparison is anchored by the canonical streaming baseline. Each cell reports image-level \acs{AUROC} and, in parentheses, pixel-level \acs{AUROC}. The vertically separated Original PatchCore column is an offline reference and is excluded from boldface selection.}
  \label{tab:mvtec-kcenter-budget-auroc}
\end{table}

\begin{table}[htbp]
  \centering
  \scriptsize
  \resizebox{\textwidth}{!}{%
  \begin{tabular}{@{}l*{4}{c}|c|@{}}
    \toprule
    Dataset &
    \shortstack{Streaming\\Mahalanobis\\PatchCore} &
    \shortstack{Streaming\\Mahalanobis\\+ \acs{GeoReS}\\$K=1000,\alpha=0.95$} &
    \shortstack{Streaming\\Mahalanobis\\+ \acs{GeoReS}\\$K=2000,\alpha=0.90$} &
    \shortstack{Streaming\\Mahalanobis\\+ \acs{GeoReS}\\$K=5000,\alpha=0.90$} &
    \shortstack{Original\\PatchCore} \\
    \midrule
    Bottle & \shortstack{\textbf{1.000}\\(0.984)} & \shortstack{\textbf{1.000}\\(\textbf{0.985})} & \shortstack{0.999\\(\textbf{0.985})} & \shortstack{\textbf{1.000}\\(\textbf{0.985})} & \shortstack{1.000\\(0.986)} \\
    \tabrowsep
    Cable & \shortstack{0.969\\(0.968)} & \shortstack{0.975\\(0.969)} & \shortstack{\textbf{0.976}\\(0.968)} & \shortstack{0.964\\(\textbf{0.972})} & \shortstack{0.985\\(0.981)} \\
    \tabrowsep
    Capsule & \shortstack{0.994\\(0.987)} & \shortstack{\textbf{0.995}\\(\textbf{0.989})} & \shortstack{0.993\\(\textbf{0.989})} & \shortstack{0.990\\(\textbf{0.989})} & \shortstack{0.990\\(0.989)} \\
    \tabrowsep
    Carpet & \shortstack{\textbf{0.985}\\(0.987)} & \shortstack{0.983\\(0.987)} & \shortstack{0.983\\(\textbf{0.988})} & \shortstack{0.983\\(\textbf{0.988})} & \shortstack{0.980\\(0.988)} \\
    \tabrowsep
    Grid & \shortstack{\textbf{1.000}\\(0.985)} & \shortstack{\textbf{1.000}\\(0.985)} & \shortstack{0.999\\(0.985)} & \shortstack{\textbf{1.000}\\(\textbf{0.986})} & \shortstack{0.997\\(0.984)} \\
    \tabrowsep
    Hazelnut & \shortstack{\textbf{1.000}\\(0.984)} & \shortstack{\textbf{1.000}\\(0.985)} & \shortstack{\textbf{1.000}\\(0.985)} & \shortstack{\textbf{1.000}\\(\textbf{0.986})} & \shortstack{1.000\\(0.986)} \\
    \tabrowsep
    Leather & \shortstack{\textbf{1.000}\\(\textbf{0.992})} & \shortstack{\textbf{1.000}\\(\textbf{0.992})} & \shortstack{\textbf{1.000}\\(\textbf{0.992})} & \shortstack{\textbf{1.000}\\(\textbf{0.992})} & \shortstack{1.000\\(0.992)} \\
    \tabrowsep
    Metal nut & \shortstack{\textbf{1.000}\\(0.988)} & \shortstack{\textbf{1.000}\\(0.988)} & \shortstack{\textbf{1.000}\\(\textbf{0.989})} & \shortstack{\textbf{1.000}\\(0.988)} & \shortstack{1.000\\(0.985)} \\
    \tabrowsep
    Pill & \shortstack{0.966\\(\textbf{0.982})} & \shortstack{\textbf{0.978}\\(0.981)} & \shortstack{0.959\\(0.976)} & \shortstack{0.968\\(0.981)} & \shortstack{0.961\\(0.976)} \\
    \tabrowsep
    Screw & \shortstack{0.937\\(0.982)} & \shortstack{0.975\\(0.987)} & \shortstack{0.978\\(0.991)} & \shortstack{\textbf{0.981}\\(\textbf{0.993})} & \shortstack{0.984\\(0.994)} \\
    \tabrowsep
    Tile & \shortstack{\textbf{0.991}\\(\textbf{0.957})} & \shortstack{\textbf{0.991}\\(0.956)} & \shortstack{\textbf{0.991}\\(\textbf{0.957})} & \shortstack{\textbf{0.991}\\(\textbf{0.957})} & \shortstack{0.990\\(0.960)} \\
    \tabrowsep
    Toothbrush & \shortstack{\textbf{1.000}\\(0.987)} & \shortstack{\textbf{1.000}\\(0.987)} & \shortstack{\textbf{1.000}\\(0.987)} & \shortstack{\textbf{1.000}\\(\textbf{0.988})} & \shortstack{1.000\\(0.987)} \\
    \tabrowsep
    Transistor & \shortstack{\textbf{0.999}\\(0.943)} & \shortstack{0.995\\(0.943)} & \shortstack{0.997\\(0.945)} & \shortstack{0.994\\(\textbf{0.949})} & \shortstack{1.000\\(0.955)} \\
    \tabrowsep
    Wood & \shortstack{\textbf{0.995}\\(0.950)} & \shortstack{0.993\\(0.950)} & \shortstack{0.992\\(\textbf{0.951})} & \shortstack{0.991\\(0.949)} & \shortstack{0.990\\(0.950)} \\
    \tabrowsep
    Zipper & \shortstack{0.997\\(\textbf{0.989})} & \shortstack{0.997\\(\textbf{0.989})} & \shortstack{\textbf{0.998}\\(\textbf{0.989})} & \shortstack{0.995\\(\textbf{0.989})} & \shortstack{0.994\\(0.988)} \\
    \tabrowsep
    Mean & \shortstack{0.989\\(0.978)} & \shortstack{\textbf{0.992}\\(0.978)} & \shortstack{0.991\\(0.978)} & \shortstack{0.990\\(\textbf{0.980})} & \shortstack{0.991\\(0.980)} \\
    \bottomrule
  \end{tabular}}
  \caption{\acs{MVTecAD} category-wise \acs{GeoReS} ablation in \acs{AUROC}. Each cell reports image-level \acs{AUROC} and, in parentheses, pixel-level \acs{AUROC}. The vertically separated Original PatchCore column is an offline reference and is excluded from boldface selection.}
  \label{tab:mvtec-geores-ablation-auroc}
\end{table}

\begin{table}[htbp]
  \centering
  \scriptsize
  \resizebox{\textwidth}{!}{%
  \begin{tabular}{@{}l*{5}{c}|c|@{}}
    \toprule
    Dataset &
    \shortstack{Streaming\\Mahalanobis\\PatchCore\\$K=1000,m_c=256$} &
    \shortstack{Streaming\\Mahalanobis\\$K=1000,m_c=1024$} &
    \shortstack{Streaming\\Mahalanobis\\$K=5000,m_c=256$} &
    \shortstack{Streaming\\Mahalanobis\\$K=5000,m_c=1024$} &
    \shortstack{Streaming\\Mahalanobis\\$K=5000,m_c=4096$} &
    \shortstack{Original\\PatchCore} \\
    \midrule
    Meniscus & 0.974 & 0.976 & 0.975 & 0.982 & \textbf{0.988} & 0.992 \\
    \tabrowsep
    Bottom & 0.995 & 0.997 & 0.995 & 0.997 & \textbf{0.998} & 0.998 \\
    \tabrowsep
    Lyo & 0.988 & 0.987 & 0.984 & 0.989 & \textbf{0.994} & 0.998 \\
    \tabrowsep
    Mean & 0.986 & 0.987 & 0.985 & 0.990 & \textbf{0.993} & 0.996 \\
    \bottomrule
  \end{tabular}}
  \caption{Industrial dataset-wise \(k\)-center budget ablation in \acs{AUROC}. The comparison is anchored by the canonical streaming baseline. Values are image-level \acs{AUROC}. The vertically separated Original PatchCore column is an offline reference and is excluded from boldface selection.}
  \label{tab:industrial-kcenter-budget-auroc}
\end{table}

\begin{table}[htbp]
  \centering
  \scriptsize
  \resizebox{\textwidth}{!}{%
  \begin{tabular}{@{}l*{4}{c}|c|@{}}
    \toprule
    Dataset &
    \shortstack{Streaming\\Mahalanobis\\PatchCore} &
    \shortstack{Streaming\\Mahalanobis\\+ \acs{GeoReS}\\$K=1000,\alpha=0.95$} &
    \shortstack{Streaming\\Mahalanobis\\+ \acs{GeoReS}\\$K=2000,\alpha=0.90$} &
    \shortstack{Streaming\\Mahalanobis\\+ \acs{GeoReS}\\$K=5000,\alpha=0.90$} &
    \shortstack{Original\\PatchCore} \\
    \midrule
    Meniscus & 0.974 & 0.982 & 0.988 & \textbf{0.990} & 0.992 \\
    \tabrowsep
    Bottom & 0.995 & 0.996 & 0.997 & \textbf{0.998} & 0.998 \\
    \tabrowsep
    Lyo & 0.988 & 0.986 & 0.987 & \textbf{0.995} & 0.998 \\
    \tabrowsep
    Mean & 0.986 & 0.988 & 0.991 & \textbf{0.994} & 0.996 \\
    \bottomrule
  \end{tabular}}
  \caption{Industrial dataset-wise \acs{GeoReS} ablation in \acs{AUROC}. Values are image-level \acs{AUROC}. The vertically separated Original PatchCore column is an offline reference and is excluded from boldface selection.}
  \label{tab:industrial-geores-ablation-auroc}
\end{table}

\clearpage

\subsection{Resource Footprint and Runtime}
\label{sec:results:memory}

Table~\ref{tab:mvtec-resource} isolates the efficiency profile of the same four
configurations used in Table~\ref{tab:mvtec-auroc}.
The accuracy ranking is established in that earlier benchmark table; here the
global summary is limited to mean training time, image-weighted per-image
inference time, and resident memory.
Because \acs{MVTecAD} is a collection of separate category datasets, a single
efficiency aggregate can still hide which categories drive the operational
cost.
Tables~\ref{tab:mvtec-ram} and~\ref{tab:mvtec-inference-per-image-time}
therefore report the corresponding category-resolved peak resident memory and
per-image inference time, with an image-weighted mean summary.
On \acs{MVTecAD}, the bounded-memory variants pay additional fit-time overhead even
as they reduce \acs{RAM}, because the benchmark is still small enough that the
offline reference remains the fastest configuration to build.
The offline reference defines the highest-memory endpoint on this benchmark,
with a \(5.41\)~GB peak.
Streaming \acs{MHPC} reduces that peak to \(2.78\)~GB while staying
in the same operational regime as the other bounded-memory variants.
The Euclidean streaming control remains the lightest at inference, whereas the
covariance-aware baseline carries a moderate latency premium.
The slower Original PatchCore inference is consistent with querying a larger
retained patch support: the offline reference keeps the largest memory
footprint in this table, whereas the streaming variants cap the candidate bank
before nearest-neighbour lookup.
This subsection therefore quantifies the runtime and memory cost of the
preferred streaming configuration rather than restating the quality ranking
established above.

\begin{table}[htbp]
  \centering
  \scriptsize
  \resizebox{\textwidth}{!}{%
  \begin{tabular}{@{}l*{4}{c}@{}}
    \toprule
    Measure &
    \shortstack{Original\\PatchCore} &
    \shortstack{Streaming\\PatchCore\\(Euclidean)} &
    \shortstack{Streaming\\Mahalanobis\\PatchCore} &
    \shortstack{Streaming\\Mahalanobis\\+ mini-batch\\$k$-means} \\
    \midrule
    Mean training time (s) & \textbf{12.444} & 54.822 & 57.913 & 24.358 \\
    \tabrowsep
    Per-image inference time (ms/image) & 467.092 & \textbf{39.592} & 42.610 & 46.623 \\
    \tabrowsep
    Mean \acs{RAM} (GB) & 4.125 & \textbf{2.424} & 2.430 & 2.489 \\
    \tabrowsep
    Max \acs{RAM} (GB) & 5.412 & \textbf{2.611} & 2.775 & 2.768 \\
    \bottomrule
  \end{tabular}}
  \caption{\acs{MVTecAD} efficiency summary for the four primary methods.
  Mean training time and mean resident memory are macro-averaged over the 15
  category datasets, max \acs{RAM} reports the largest category peak, and per-image
  inference time is reported as an image-weighted mean. Lower is better for all
  rows.}
  \label{tab:mvtec-resource}
\end{table}

\begin{table}[htbp]
  \centering
  \scriptsize
  \resizebox{\textwidth}{!}{%
  \begin{tabular}{@{}l*{4}{c}@{}}
    \toprule
    Dataset &
    \shortstack{Original\\PatchCore} &
    \shortstack{Streaming\\PatchCore\\(Euclidean)} &
    \shortstack{Streaming\\Mahalanobis\\PatchCore} &
    \shortstack{Streaming\\Mahalanobis\\+ mini-batch\\$k$-means} \\
    \midrule
    Bottle & 3.520 & 2.164 & \textbf{2.019} & 2.060 \\
    \tabrowsep
    Cable & 3.741 & \textbf{2.289} & 2.517 & 2.372 \\
    \tabrowsep
    Capsule & 3.825 & 2.393 & \textbf{2.361} & 2.515 \\
    \tabrowsep
    Carpet & 4.428 & \textbf{2.425} & 2.536 & 2.448 \\
    \tabrowsep
    Grid & 4.304 & 2.381 & \textbf{2.376} & 2.448 \\
    \tabrowsep
    Hazelnut & 5.412 & 2.408 & \textbf{2.337} & 2.405 \\
    \tabrowsep
    Leather & 4.158 & 2.457 & \textbf{2.436} & 2.503 \\
    \tabrowsep
    Metal nut & 3.942 & \textbf{2.472} & 2.483 & 2.520 \\
    \tabrowsep
    Pill & 4.414 & 2.469 & \textbf{2.437} & 2.525 \\
    \tabrowsep
    Screw & 5.127 & \textbf{2.611} & 2.775 & 2.768 \\
    \tabrowsep
    Tile & 4.252 & 2.503 & \textbf{2.492} & 2.675 \\
    \tabrowsep
    Toothbrush & 2.557 & 2.459 & \textbf{2.405} & 2.531 \\
    \tabrowsep
    Transistor & 3.844 & \textbf{2.340} & 2.405 & 2.495 \\
    \tabrowsep
    Wood & 4.198 & 2.443 & \textbf{2.426} & 2.543 \\
    \tabrowsep
    Zipper & 4.149 & 2.541 & \textbf{2.446} & 2.531 \\
    \tabrowsep
    Mean & 4.125 & \textbf{2.424} & 2.430 & 2.489 \\
    \tabrowsep
    Max & 5.412 & \textbf{2.611} & 2.775 & 2.768 \\
    \bottomrule
  \end{tabular}}
  \caption{\acs{MVTecAD} per-category peak resident memory. The final rows report
  the mean and maximum of the category-level peaks. Values are reported in GB;
  lower is better.}
  \label{tab:mvtec-ram}
\end{table}

\begin{table}[htbp]
  \centering
  \scriptsize
  \resizebox{\textwidth}{!}{%
  \begin{tabular}{@{}l*{4}{c}@{}}
    \toprule
    Dataset &
    \shortstack{Original\\PatchCore} &
    \shortstack{Streaming\\PatchCore\\(Euclidean)} &
    \shortstack{Streaming\\Mahalanobis\\PatchCore} &
    \shortstack{Streaming\\Mahalanobis\\+ mini-batch\\$k$-means} \\
    \midrule
    Bottle & 401.656 & \textbf{44.991} & 46.320 & 48.877 \\
    \tabrowsep
    Cable & 417.472 & \textbf{39.743} & 42.693 & 47.146 \\
    \tabrowsep
    Capsule & 408.882 & \textbf{36.563} & 39.329 & 41.677 \\
    \tabrowsep
    Carpet & 516.694 & \textbf{39.022} & 41.937 & 46.754 \\
    \tabrowsep
    Grid & 495.392 & \textbf{44.855} & 47.268 & 51.272 \\
    \tabrowsep
    Hazelnut & 711.642 & \textbf{44.129} & 48.588 & 51.903 \\
    \tabrowsep
    Leather & 451.504 & \textbf{34.042} & 36.994 & 45.095 \\
    \tabrowsep
    Metal nut & 408.750 & \textbf{39.849} & 43.035 & 45.787 \\
    \tabrowsep
    Pill & 485.295 & \textbf{30.986} & 34.048 & 39.871 \\
    \tabrowsep
    Screw & 578.678 & \textbf{31.495} & 35.025 & 40.908 \\
    \tabrowsep
    Tile & 428.423 & \textbf{44.118} & 47.742 & 47.159 \\
    \tabrowsep
    Toothbrush & 178.059 & \textbf{76.339} & 78.436 & 86.423 \\
    \tabrowsep
    Transistor & 406.183 & \textbf{46.204} & 49.858 & 53.245 \\
    \tabrowsep
    Wood & 471.315 & \textbf{48.865} & 52.062 & 60.711 \\
    \tabrowsep
    Zipper & 439.381 & \textbf{33.038} & 35.227 & 35.010 \\
    \tabrowsep
    Image-weighted mean & 467.092 & \textbf{39.592} & 42.610 & 46.623 \\
    \bottomrule
  \end{tabular}}
  \caption{\acs{MVTecAD} per-category per-image inference time. The final row
  reports the image-weighted mean across all test images. Values are reported
  in ms/image; lower is better.}
  \label{tab:mvtec-inference-per-image-time}
\end{table}

Tables~\ref{tab:industrial-resource}, \ref{tab:industrial-ram},
and~\ref{tab:industrial-inference-per-image-time} are the industrial
counterparts of the same efficiency views.
As in the earlier industrial image-level results, they should be read as
telemetry summaries rather than as a second ranking of image metrics.
The industrial comparison makes the memory contrast sharper: the offline
reference reaches \(37.43\)~GB peak \acs{RAM}, whereas the canonical
bounded-memory alternative remains at \(8.23\)~GB.
In this larger setting, the same bounded-memory construction also reduces mean
training time relative to the offline reference rather than increasing it
(\(1079\)~s versus \(14837\)~s).
Within that bounded-memory regime, the covariance-aware streaming baseline and
the Euclidean control stay close in memory while differing mainly in inference
cost and the earlier benchmark quality ranking.
Figure~\ref{fig:industrial-auroc-ram-pareto} summarises the same trade-off for
the primary methods, \(k\)-center budget configurations, and \ac{GeoReS}
ablations.

\begin{table}[htbp]
  \centering
  \scriptsize
  \resizebox{\textwidth}{!}{%
  \begin{tabular}{@{}lcccc@{}}
    \toprule
    Measure &
    \shortstack{Original\\PatchCore} &
    \shortstack{Streaming\\PatchCore\\(Euclidean)} &
    \shortstack{Streaming\\Mahalanobis\\PatchCore} &
    \shortstack{Streaming\\Mahalanobis\\+ mini-batch\\$k$-means} \\
    \midrule
    Mean training time (s) & 14837.300 & 939.029 & 1078.977 & \textbf{607.680} \\
    \tabrowsep
    Per-image inference time (ms/image) & 8207.638 & 64.501 & \textbf{60.702} & 68.824 \\
    \tabrowsep
    Mean \acs{RAM} (GB) & 31.462 & 8.845 & \textbf{8.001} & 8.571 \\
    \tabrowsep
    Max \acs{RAM} (GB) & 37.430 & 9.656 & \textbf{8.227} & 9.052 \\
    \bottomrule
  \end{tabular}}
  \caption{Industrial efficiency summary for the four primary methods on the
  selected industrial datasets.
  Mean training time and mean resident memory are macro-averaged over the
  industrial datasets, max \acs{RAM} reports the largest dataset peak, and per-image
  inference time is reported as an image-weighted mean. Lower is better for all
  rows.}
  \label{tab:industrial-resource}
\end{table}

\begin{table}[htbp]
  \centering
  \scriptsize
  \resizebox{\textwidth}{!}{%
  \begin{tabular}{@{}lcccc@{}}
    \toprule
    Dataset &
    \shortstack{Original\\PatchCore} &
    \shortstack{Streaming\\PatchCore\\(Euclidean)} &
    \shortstack{Streaming\\Mahalanobis\\PatchCore} &
    \shortstack{Streaming\\Mahalanobis\\+ mini-batch\\$k$-means} \\
    \midrule
    Meniscus & 37.430 & 9.084 & \textbf{8.098} & 8.417 \\
    \tabrowsep
    Bottom & 26.666 & 9.656 & \textbf{8.227} & 9.052 \\
    \tabrowsep
    Lyo & 30.289 & 7.795 & \textbf{7.677} & 8.246 \\
    \tabrowsep
    Mean & 31.462 & 8.845 & \textbf{8.001} & 8.571 \\
    \tabrowsep
    Max & 37.430 & 9.656 & \textbf{8.227} & 9.052 \\
    \bottomrule
  \end{tabular}}
  \caption{Industrial peak resident memory on the selected industrial
  datasets. The final rows report the mean and maximum of the dataset-level
  peaks. Values are reported in GB; lower is better.}
  \label{tab:industrial-ram}
\end{table}

\begin{table}[htbp]
  \centering
  \scriptsize
  \resizebox{\textwidth}{!}{%
  \begin{tabular}{@{}lcccc@{}}
    \toprule
    Dataset &
    \shortstack{Original\\PatchCore} &
    \shortstack{Streaming\\PatchCore\\(Euclidean)} &
    \shortstack{Streaming\\Mahalanobis\\PatchCore} &
    \shortstack{Streaming\\Mahalanobis\\+ mini-batch\\$k$-means} \\
    \midrule
    Meniscus & 9889.752 & 67.924 & \textbf{65.842} & 72.394 \\
    \tabrowsep
    Bottom & 6893.612 & 60.215 & \textbf{55.843} & 64.659 \\
    \tabrowsep
    Lyo & 7726.831 & 65.657 & \textbf{60.352} & 69.626 \\
    \tabrowsep
    Image-weighted mean & 8207.638 & 64.501 & \textbf{60.702} & 68.824 \\
    \bottomrule
  \end{tabular}}
  \caption{Industrial per-dataset per-image inference time on the selected
  industrial datasets. The final row reports the image-weighted mean across
  all test images. Values are reported in ms/image; lower is better.}
  \label{tab:industrial-inference-per-image-time}
\end{table}

\begin{figure}[htbp]
  \centering
  \includegraphics[width=0.92\textwidth]{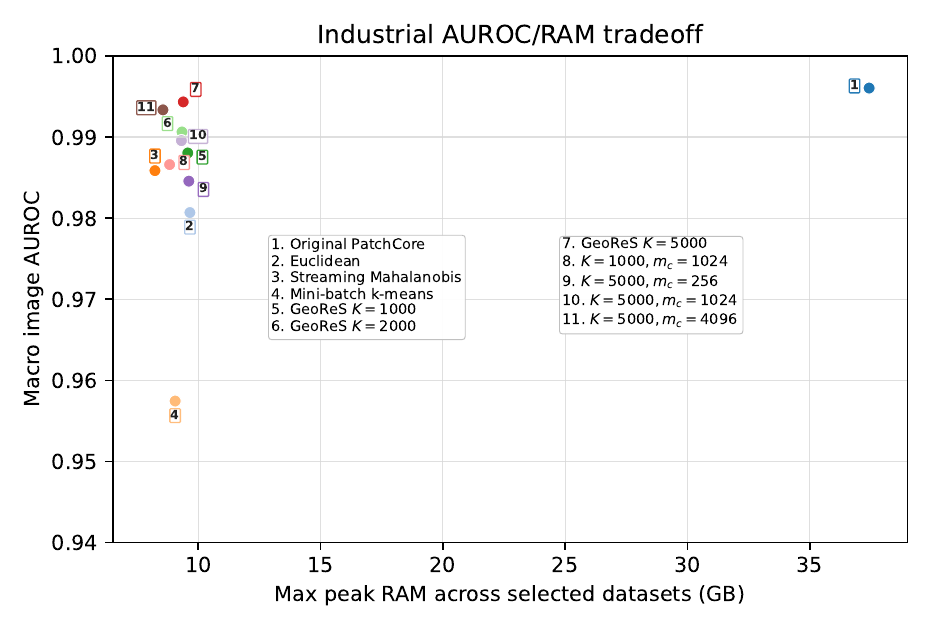}
  \caption{Industrial image-\acs{AUROC}/\acs{RAM} trade-off across the selected
  industrial datasets. Each point reports macro image \acs{AUROC} and the
  largest dataset-level peak \acs{RAM} for one reported method or ablation.}
  \label{fig:industrial-auroc-ram-pareto}
\end{figure}
\clearpage

\section{Discussion}
\label{sec:discussion}

\subsection{When Mahalanobis Distance Helps}
The Mahalanobis whitening transform is most useful when Euclidean distances in
the reduced feature space over- or under-weight directions because normal patch
embeddings are anisotropic and correlated.
The results support this interpretation, but they also make its scope precise.
On \acs{MVTecAD}, switching from the Euclidean streaming control to Streaming
\acs{MHPC} raises mean image \ac{AUROC} from \(0.979\) to \(0.989\).
This closes much of the image-level gap to offline PatchCore, whose
corresponding mean \ac{AUROC} is \(0.991\).
At the same time, the MVTec pixel-level \ac{AUROC} summary changes only from
\(0.965\) to \(0.978\) relative to the Euclidean streaming control and remains
slightly below the offline reference at \(0.980\).
Thus, on the public benchmark, covariance-aware retrieval primarily improves
image-level discrimination, while localisation remains a completeness result
rather than the central claim.

The selected industrial results show a stronger version of the same effect.
There, Mahalanobis retrieval improves the Euclidean streaming control from
mean \ac{AUROC} \(0.981\) to \(0.986\).
The \ac{GeoReS} \(K=2000\) ablation then raises mean \ac{AUROC} to \(0.991\),
while the high-budget \(K=5000\) configuration reaches \(0.994\).
This is consistent with the qualitative structure of the selected industrial
tasks, which cover distinct but comparably challenging inspection regimes.
Meniscus contains variable meniscus and fluid patterns caused by turbulence,
reflections, and local appearance changes. Bottom contains bottom-region
variability in amber glass ampoules with viscous liquid, where heavy glass
particles may be difficult to separate from nominal appearance. Lyo covers
lyophilised-cake product inspection, where relevant anomalies include visible
foreign particles, missing cake, and incorrect cake volume.
In these settings, covariance-aware scaling and residual-tail coverage can make
the normal support more useful for image-level retrieval, especially when the
defect evidence is small, localised, or subtle, and the corresponding nominal
appearance modes are sparsely sampled in the normal support.
The evidence should still be read conservatively, because the industrial suite
is small and image-level only.
Even so, it indicates that the proposed geometry and coverage mechanisms are
most useful when industrial defects are subtle or sparsely represented and when
training-time memory is constrained.

\subsection{Streaming-Compatible Training Trade-Offs}
The word ``streaming'' is used here in a bounded-memory training sense.
\acs{MHPC} consumes descriptors through mini-batches and does not require the
full normal patch matrix to be materialised, but it is not a strict online
one-pass learner.
The reduction, covariance, and memory-bank stages are fitted over separate
passes through a re-iterable training set; the \ac{GeoReS} ablation adds an
extra memory-bank pass because it first estimates residual coverage.

This distinction matters for the interpretation of the resource results.
The method trades additional sequential scans of the training set for a much
smaller peak resident memory footprint.
It is therefore most appropriate when the full feature pool is the operational
bottleneck, or when normal data are already curated in a dataset that can be
rescanned during model preparation.
Additional curated normal batches can be incorporated by continuing the same
fitting pipeline, but if the reduction or covariance transform changes, the
stored support set must remain consistent with the updated transformed space.
Under this interpretation, the term streaming captures the intended engineering
property: bounded-memory construction and extension of a PatchCore-style
support set, not autonomous online adaptation.

\subsection{Support-Set Construction and Budget}
The empirical comparison between \(k\)-center and mini-batch \(k\)-means is not
only a hyperparameter study, but a contrast between two different summarisation
principles.
\(k\)-center preserves representative support points chosen for coverage,
whereas mini-batch \(k\)-means replaces them with centroids.
The pure-\(k\)-center budget study then separates two distinct levers:
the local merge-reduce chunk size \(m_c\), which controls how much information is
retained before each reduction step, and the final bank budget \(K\), which
controls how many representatives survive globally.
On \acs{MVTecAD}, increasing \(m_c\) at fixed \(K=1000\) gives the clearest small
gain, matching the offline PatchCore mean \ac{AUROC} after rounding, but it
also raises mean training time substantially.
Increasing the final bank to \(K=5000\) does not improve the macro result in the
same way, which suggests that the final support set is not the dominant
bottleneck for this benchmark.
The industrial budget ladder shows the complementary side of the same
trade-off: the high-budget \(K=5000,m_c=4096\) configuration approaches the
larger \ac{GeoReS} result while retaining a much smaller \acs{RAM} footprint than
offline PatchCore, \(8.55\)~GB versus \(37.43\)~GB.
\ac{GeoReS} confirms the same picture from another angle: residual-tail
selection helps selected categories, especially \emph{screw}, but its aggregate
benefit is modest relative to the extra pass and higher inference cost.
For the reported MVTec configuration, the merge-reduce \(k\)-center bank is
therefore a pragmatic operating point, not merely a low-budget
approximation.

\subsection{Limitations}
A fully autonomous post-deployment continual-learning protocol remains outside
the scope of this study.
In particular, updating the reduction or Cholesky factor after deployment would
require the memory bank to be rebuilt or consistently transformed so that all
stored references remain in the same feature space.
Furthermore, the streaming variant assumes that the training set can be iterated
at least three times, which may not be feasible in strict single-pass scenarios.
Only \acs{MVTecAD} provides pixel-level masks, so localisation-oriented evaluation
can be studied publicly only on that benchmark; the selected industrial
datasets are instead assessed through image-level discrimination and resource
telemetry.
The accuracy of the streaming variant may also lag behind the original batch
PatchCore baseline on some categories because staged reduction and
merge-reduce support-bank construction do not optimise over the full patch pool
in one global batch.
Section~\ref{sec:results:memory} characterises this trade-off empirically.

\section{Conclusion}
\label{sec:conclusion}

We have presented \acs{MHPC}, a PatchCore extension that combines
covariance-aware retrieval with bounded-memory bank construction.
The method keeps the retrieval structure of PatchCore, but performs nearest
neighbour search after reduction, regularised covariance estimation, and
whitening; it also replaces full-pool bank construction with a merge-reduce
\(k\)-center summary that makes training-time memory explicit.

The empirical evidence is consistent with this design goal.
On \acs{MVTecAD}, evaluated as 15 category-specific datasets, Streaming
\acs{MHPC} remains close to the offline PatchCore reference at image level, with
mean \ac{AUROC} \(0.989\) versus \(0.991\), while reducing peak \acs{RAM} from
\(5.41\) to \(2.78\)~GB.
On the selected industrial datasets, the covariance-aware streaming variant
improves the Euclidean streaming control from \(0.981\) to \(0.986\) mean image
\ac{AUROC}, and the two-pass \ac{GeoReS} ablation reaches \(0.991\) at
\(K=2000\) and \(0.994\) at the higher-budget \(K=5000\) setting.
The canonical streaming configuration also reduces the industrial peak \acs{RAM} from
the offline PatchCore \(37.43\)~GB footprint to \(8.23\)~GB.
These results support a deliberately narrow conclusion: \acs{MHPC} is not a
universal replacement for full-memory PatchCore in every metric, but it is a
practical bounded-memory formulation that preserves most benchmark accuracy and
becomes especially useful when industrial normal data contain locally
under-represented nominal appearance modes.

Future work should focus on extending the industrial evidence with additional
production datasets and on collecting localisation labels or deployment streams
that allow calibration, localisation behaviour, and continual updates to be
evaluated directly.

\section*{Acknowledgments}
\label{sec:acknowledgments}
Part of this work was carried out during the first author's industrial PhD at
the University of Ferrara. The second author contributed while completing his
master's degree at the University of Ferrara.

\bibliographystyle{elsarticle-num}
\begingroup
\sloppy
\hbadness=10000
\bibliography{biblio}
\endgroup

\appendix

\section{Notation Summary}
\label{app:symbols}

\begingroup
\small
\setlength{\tabcolsep}{4pt}
\begin{tabular}{@{}>{\raggedright\arraybackslash}p{0.18\textwidth} >{\raggedright\arraybackslash}p{0.73\textwidth}@{}}
\toprule
\textbf{Symbol} & \textbf{Meaning} \\
\midrule
$x$ & Input image. \\
$\PP(x)=\{u^{(j)}\}_{j=1}^{n_x}$ & Patch set extracted from image $x$. \\
$u \in \mathbb{R}^{d_0}$ & Patch embedding before reduction. \\
$R$ & Fitted dimensionality reducer, instantiated by \acs{IPCA} in the streaming baseline. \\
$W \in \mathbb{R}^{d_0\times k}$ & Projection matrix in the centred \acs{PCA}/\acs{IPCA} reducer \(R(u)=W^\top(u-\bar{u}_R)\). \\
$\bar{u}_R$ & Mean used by the fitted \acs{PCA}/\acs{IPCA} reducer. \\
$\tilde{u} = R(u) \in \mathbb{R}^{k}$ & Reduced patch embedding. \\
$\muh$ & Mean of reduced embeddings estimated during streaming covariance fitting. \\
$\Sigma_{\mathrm{reg}}$ & Regularised covariance matrix in reduced space. \\
$L$ & Cholesky factor such that
$\Sigma_L=\Sigma_{\mathrm{reg}}+\delta I=LL^\top$. \\
$z = L^{-1}(\tilde{u}-\muh)$ & Whitened embedding used for retrieval. \\
$\M$ & Final memory bank used at inference. \\
$K$ & Memory-bank budget. \\
$m_c$ & Local coreset size in merge-reduce $k$-centre aggregation. \\
$P$ & Provisional \acs{GeoReS} coverage set. \\
$K_0, K_t$ & Main and residual-tail budgets in \acs{GeoReS}. \\
$\alpha$ & \acs{GeoReS} main-budget fraction used to set \(K_0\). \\
$q$ & Number of high-residual candidates retained in \acs{GeoReS}. \\
$r_P(z)$ & Residual of \(z\) with respect to provisional bank \(P\). \\
$T_q$ & Candidate set containing the \(q\) largest-residual normal patches. \\
$\delta(u,m)$ & Squared retrieval distance between transformed query patch \(z(u)\) and bank element \(m\). \\
$\tau(x)$ & Stabilising offset used in the Eq.~\eqref{eq:score}-style reweighting. \\
$w(x)$ & Local-support reweighting factor applied to the maximal patch score. \\
\bottomrule
\end{tabular}
\endgroup

\section{Whitening and Mahalanobis Equivalence}
\label{app:whitening-equivalence}

This appendix states formally the equivalence used in
Section~\ref{sec:method:whitening-regularization}.
Here \(d_{\mathrm{MH}}\) is used only as a compact mathematical symbol for
Mahalanobis distance.

\begin{lemma}[Mahalanobis--Euclidean equivalence via whitening]
\label{lem:whitening-mahalanobis}
Let $\Sigma \in \mathbb{R}^{d\times d}$ with $\Sigma \succ 0$, and
let $x,y \in \mathbb{R}^{d}$.
Then there exists an invertible matrix $B$ such that
\begin{equation}
  d_{\mathrm{MH}}(x,y) = \|B(x-y)\|_2.
\end{equation}
In particular, if $\Sigma = CC^\top$ is the Cholesky factorisation of
$\Sigma$ (with $C$ lower triangular and positive diagonal), then
\begin{equation}
  d_{\mathrm{MH}}(x,y) = \|C^{-1}(x-y)\|_2.
\end{equation}
\end{lemma}
\begin{proof}
Since $\Sigma \succ 0$, the Cholesky factorisation $\Sigma=CC^\top$ exists and
$C$ is invertible.
Hence $\Sigma^{-1}=C^{-\top}C^{-1}$, and therefore
\begin{align}
  d_{\mathrm{MH}}^2(x,y)
  &= (x-y)^\top \Sigma^{-1}(x-y) \nonumber\\
  &= (x-y)^\top C^{-\top}C^{-1}(x-y) \nonumber\\
  &= \|C^{-1}(x-y)\|_2^2. \label{eq:mh_lemma}
\end{align}
Taking square roots proves
$d_{\mathrm{MH}}(x,y)=\|C^{-1}(x-y)\|_2$.
Geometrically, Mahalanobis distance is Euclidean distance in the whitened
space $z=C^{-1}x$ (equivalently $z=\Sigma^{-1/2}x$).
\end{proof}

\section{Covariance Regularisation for Numerical Stability}
\label{app:covreg-alternatives}

This appendix collects the covariance-regularisation details that support the
whitening step in Section~\ref{sec:method:whitening-regularization}.
Their purpose is to make the numerical assumptions explicit rather than to
define an additional ablation family.
All reported \acs{MHPC} baselines use fixed shrinkage with numerical
stabilisation; the analytical shrinkage rules below are standard reference
policies for covariance estimation and are not part of the experimental
comparison in this paper.
Given an empirical covariance matrix $\hat{\Sigma} \in \mathbb{R}^{d\times d}$,
the common shrinkage form is
\begin{equation}
  \Sigma_{\mathrm{shr}} = (1-\lambda)\hat{\Sigma} + \lambda \bar{\sigma} I,
  \qquad
  \bar{\sigma} = \frac{\operatorname{tr}(\hat{\Sigma})}{d},
  \qquad
  \lambda \in [0,1].
\end{equation}

\paragraph{\acs{OAS}}
\acl{OAS}~\cite{chen2010shrinkage} selects $\lambda$
analytically:
\begin{equation}
  \lambda_{\mathrm{OAS}} =
  \operatorname{clip}_{[0,1]}\!\left(
    \frac{(1-\tfrac{2}{d})\operatorname{tr}(\hat{\Sigma}^2)
          + \operatorname{tr}(\hat{\Sigma})^2}
         {(N+1-\tfrac{2}{d})
          \bigl(\operatorname{tr}(\hat{\Sigma}^2)
                - \tfrac{\operatorname{tr}(\hat{\Sigma})^2}{d}\bigr)}
  \right). \label{eq:oas}
\end{equation}
It is designed to improve conditioning when the sample size is limited relative
to the feature dimension.

\paragraph{Ledoit-Wolf / \ac{RBLW}}
Ledoit-Wolf~\cite{ledoit2004well} provides an analytical shrinkage rule that
minimises expected Frobenius loss.
In strict streaming settings, where the full feature matrix is not retained,
the \acl{RBLW} approximation can be used:
\begin{equation}
  \lambda_{\mathrm{RBLW}} =
  \operatorname{clip}_{[0,1]}\!\left(
    \frac{(\tfrac{N-2}{N})\operatorname{tr}(\hat{\Sigma}^2)
          + \operatorname{tr}(\hat{\Sigma})^2}
         {(N+2)\bigl(\operatorname{tr}(\hat{\Sigma}^2)
                     - \tfrac{\operatorname{tr}(\hat{\Sigma})^2}{d}\bigr)}
  \right). \label{eq:rblw}
\end{equation}

\paragraph{Fixed Shrinkage}
A user-specified constant $\lambda \in [0,1]$ is used for all categories and
runs, providing deterministic behaviour and direct control over regularisation
strength.
This is the policy adopted by the reported baseline in this paper
($\lambda=0.07$).

\paragraph{Jitter-only}
No shrinkage is applied ($\lambda=0$), and positive-definiteness is enforced
only by adaptive diagonal jitter during Cholesky factorisation.
This limiting case clarifies the role of shrinkage, but it is generally less
robust when $\hat{\Sigma}$ is poorly conditioned and lies outside the reported
ablation set.

\paragraph{Eigenvalue Floor}
Independently of the shrinkage policy, an optional eigenvalue floor can be
applied to $\Sigma_{\mathrm{shr}}$:
\begin{equation}
  \gamma_i \leftarrow \max(\gamma_i,\epsilon),
  \qquad
  \epsilon = \epsilon_{\mathrm{rel}}\max\{|\bar{\sigma}|,10^{-12}\},
\end{equation}
where $\{\gamma_i\}$ are the eigenvalues of $\Sigma_{\mathrm{shr}}$ and
\(\bar{\sigma}=d^{-1}\operatorname{tr}(\Sigma_{\mathrm{shr}})\).
This produces the final matrix $\Sigma_{\mathrm{reg}}$ used in the whitening
step.

\paragraph{Adaptive jittered Cholesky}
Even after shrinkage and eigenvalue flooring, numerical issues may prevent a
successful Cholesky factorisation.
The baseline therefore computes
\begin{equation}
  L = \operatorname{chol}(\Sigma_{\mathrm{reg}} + \delta I),
\end{equation}
with diagonal jitter selected from the geometric sequence
\begin{equation}
  \delta \in \{0,\;\delta_{\min},\;\delta_{\min}m,\;\delta_{\min}m^2,\;\dots\},
\end{equation}
retrying until factorisation succeeds or until $\delta > \delta_{\max}$.
In the canonical configuration we use
\begin{equation}
  \delta_{\min}=10^{-12}, \qquad m=10, \qquad \delta_{\max}=1.
\end{equation}

\section{Memory-Bank Constructors}
\label{app:aggregation-alternatives}

This appendix defines the memory-bank constructors used in the experiments.
Let
\begin{equation}
  Z = \{z_n\}_{n=1}^{N}, \qquad z_n \in \mathbb{R}^{k},
\end{equation}
and let $K$ be the target memory-bank budget.
For streaming constructors, \(Z\) is the conceptual concatenation of the
transformed mini-batches seen during training; only the offline reference
requires materialising the full set before compression.
Each strategy defines
\begin{equation}
  \M = \mathcal{G}(Z), \qquad |\M| \le K.
\end{equation}

\paragraph{Offline greedy coreset}
The original PatchCore reference materialises the full normal patch pool and
then applies greedy farthest-first coreset selection~\cite{Roth2021}.
Starting from an initial center set \(C_1\), each subsequent representative is
chosen as
\begin{equation}
  z^\star \in
  \operatorname*{arg\,max}_{z\in Z}
  \min_{c\in C_\ell}\|z-c\|_2,
  \qquad
  C_{\ell+1}=C_\ell\cup\{z^\star\}.
\end{equation}
After \(K\) selections, the memory bank is \(\M=C_K\).
This operator is the offline reference because it requires access to the full
patch pool before compression.

\paragraph{Merge-reduce \(k\)-center}
The canonical streaming constructor uses the same coverage principle under a
bounded-memory merge-reduce scheme.
For any candidate set \(C\), define
\begin{equation}
  d_{\min}(z,C)=\min_{c\in C}\|z-c\|_2.
\end{equation}
Each stream chunk \(Z_t\) is reduced to a local farthest-first summary
\(C_t\) of size at most \(m_c\).
Pairs of summaries are then recursively merged and reduced,
\begin{equation}
  C_{\ell+1}=
  \operatorname{Reduce}\!\left(C_{\ell}^{(1)}\cup C_{\ell}^{(2)},m_c\right),
\end{equation}
where \(\operatorname{Reduce}\) denotes greedy farthest-first selection under
the same distance.
The final reduction produces \(|\M|\le K\) representatives.
Unlike centroid methods, this preserves observed normal patches as retrieval
anchors.

\paragraph{Mini-batch \(k\)-means}\mbox{}\\*
This ablation uses centroids rather than observed patches.
The standard \(k\)-means objective is
\begin{equation}
  \min_{\{c_j\}_{j=1}^{K}}
  \sum_{n=1}^{N}\left\|z_n-c_{a(z_n)}\right\|_2^2,
  \qquad
  a(z_n)=\operatorname*{arg\,min}_{j}\|z_n-c_j\|_2^2.
\end{equation}
The reported streaming ablation approximates this objective by initialising a
bounded set of centroids and applying Robbins--Monro updates as mini-batches are
read~\cite{sculley2010web}: each sample is assigned to its nearest centroid,
the corresponding per-centroid count is updated, and the centroid is moved with
the resulting count-based step size.
The resulting bank contains centroids rather than observed patch descriptors,
which is why this comparison isolates the effect of the constructor geometry.

\paragraph{\acs{GeoReS}}
\acl{GeoReS} is the targeted two-pass \(k\)-center variant used in this
ablation.
It tests whether explicit residual-tail coverage improves the streaming bank.
It first builds a provisional bank
\begin{equation}
  P=\mathcal{G}_K(Z)
\end{equation}
with the merge-reduce constructor, then scores normal patches by their residual
coverage error,
\begin{equation}
  r_P(z)=\min_{p\in P}\|z-p\|_2.
\end{equation}
The final budget is split into \(K_0\) main representatives and \(K_t=K-K_0\)
tail representatives.
The main bank is
\begin{equation}
  M_0=\mathcal{G}_{K_0}(Z),
\end{equation}
and the tail candidates are the \(q\) largest-residual normal patches,
\begin{equation}
  T_q=\operatorname{Top}_{q}\{\,z\in Z:\ r_P(z)\,\}.
\end{equation}
Tail representatives are selected greedily from \(T_q\) against the current
main-plus-tail set:
\begin{equation}
  t_j \in \operatorname*{arg\,max}_{z\in T_q\setminus\{t_1,\ldots,t_{j-1}\}}
  \min_{m\in M_0\cup\{t_1,\ldots,t_{j-1}\}}\|z-m\|_2 .
\end{equation}
The final bank is the budgeted, deduplicated completion
\begin{equation}
  \M_{\mathrm{GeoReS}}
  =\Pi_K\!\left(M_0\cup\{t_1,\ldots,t_{K_t}\};Z\right).
\end{equation}
Here \(\Pi_K\) removes duplicate representatives, fills any remaining capacity
from the coverage pool, and truncates the bank to the target budget.
Because \ac{GeoReS} requires an additional pass to identify residual-heavy
normal regions, it remains an ablation rather than the canonical streaming
constructor.

\section{Welford Algorithm for Streaming Covariance}
\label{app:welford}

The Welford algorithm~\cite{welford1962note} provides a numerically stable
method for computing the mean and covariance of a data stream without storing
the full sample matrix.
For a sequence of batches $\{U'_t\}$, let $n$ denote the current sample count,
$\muh$ the running mean, and $M_2$ the running unnormalised second
central moment matrix.
For each new batch $U'_t$ with $n_b$ samples, batch mean $\muh_b$,
and batch second central moment $M_{2,b}$, the update follows the standard
parallel Welford combination rule discussed in Section~\ref{sec:method:whitening-regularization}.
The unbiased covariance estimator is $\hat{\Sigma} = M_2 / (n-1)$.
This formulation avoids catastrophic cancellation and is equivalent to the
parallel computation described in~\cite{chan1979updating}.

\end{document}